\documentclass[letterpaper]{article} 
\usepackage{aaai23}  
\usepackage{times}  
\usepackage{helvet}  
\usepackage{courier}  
\usepackage[hyphens]{url}  
\usepackage{graphicx} 
\usepackage{natbib}  
\usepackage{caption} 
\usepackage{algorithm}
\usepackage{algorithmic}

\usepackage{times}
\usepackage{soul}
\usepackage{url}
\usepackage[utf8]{inputenc}
\usepackage{graphicx}
\usepackage{amsmath}
\usepackage{amsthm}
\usepackage{booktabs}
\usepackage{bm}
\usepackage{amssymb}
\usepackage{multirow}
\usepackage{amsfonts}
\usepackage{float}
\usepackage{booktabs}
\usepackage{mathtools}
\usepackage{subfigure}
\usepackage{enumerate}
\usepackage{newfloat}
\usepackage{listings}
\DeclareCaptionStyle{ruled}{labelfont=normalfont,labelsep=colon,strut=off} 
\floatstyle{ruled}
\newfloat{listing}{tb}{lst}{}
\floatname{listing}{Listing}
\def\1{\mathbf{1}}
\def\0{\mathbf{0}}

\def\D{{\bf D}}

\def\X{{\bf X}}
\def\E{{\bf E}}

\def\x{{\bf x}}

\def\Y{{\bf Y}}
\def\y{{\bf y}}

\def\A{{\bf A}}
\def\B{{\bf B}}

\def\pphi{\bm{\phi}}
\def\SSigma{\bm{\Sigma}}
\def\PPhi{\bm{\Phi}}
\def\LLambda{\bm{\Lambda}}

\def\u{{\bf u}}
\def\U{{\bf U}}

\def\S{{\bf S}}

\def\T{{\bf T}}

\def\I{{\bf I}}
\def\g{{\bf g}}
\def\q{{\bf q}}
\def\G{{\bf G}}
\def\Z{{\bf Z}}

\def\p{{\bf p}}
\def\P{{\bf P}}
\def\Q{{\bf Q}}
\def\H{{\bf H}}

\def\tr{{\bf tr}}
\def\Var{{\bf Var}}
\def\IP{{\bf IP}}

\newtheorem{remark}{\bf Remark}
\newtheorem{definition}{\bf Definition}
\newtheorem{lemma}{\bf Lemma}
\newtheorem{theorem}{\bf Theorem}

\newtheorem{proposition}{\bf Proposition}

\newcommand*\dif{\mathop{}\!\mathrm{d}}

\DeclarePairedDelimiter{\abs}{\lvert}{\rvert}
\DeclarePairedDelimiter{\norm}{\lVert}{\rVert}
\DeclarePairedDelimiter{\ceil}{\lceil}{\rceil}

\DeclarePairedDelimiter{\prn}{\lparen}{\rparen}
\DeclarePairedDelimiter{\brk}{\lbrack}{\rbrack}

\DeclarePairedDelimiter{\ang}{\langle}{\rangle}

\allowdisplaybreaks

\title{Robust and Fast Measure of Information via Low-rank Representation}
\author{
	Yuxin Dong\textsuperscript{\rm 1,2},
	Tieliang Gong\textsuperscript{\rm 1,2}\thanks{Corresponding author.},
	Shujian Yu\textsuperscript{\rm 3},
	Hong Chen\textsuperscript{\rm 4,5},
	Chen Li\textsuperscript{\rm 1,2}
}
\affiliations{
    \textsuperscript{\rm 1}School of Computer Science and Technology, Xi'an Jiaotong University, Xi'an 710049, China\\
	\textsuperscript{\rm 2}Shaanxi Provincial Key Laboratory of Big Data Knowledge Engineering, Ministry of Education, Xi'an 710049, China\\
	\textsuperscript{\rm 3}Machine Learning Group, UiT - The Arctic University of Norway\\
	\textsuperscript{\rm 4}College of Science, Huazhong Agriculture University, Wuhan 430070, China\\
    \textsuperscript{\rm 5}Engineering Research Center of Intelligent Technology for Agriculture, Ministry of Education, Wuhan 430070, China\\
	dongyuxin@stu.xjtu.edu.cn, adidasgtl@gmail.com, yusj9011@gmail.com, chenh@mail.hzau.edu.cn, cli@xjtu.edu.cn

}

\usepackage{bibentry}

\begin{document}

\maketitle

\begin{abstract}
    The matrix-based R\'enyi's entropy allows us to directly quantify information measures from given data, without explicit estimation of the underlying probability distribution. This intriguing property makes it widely applied in statistical inference and machine learning tasks. However, this information theoretical quantity is not robust against noise in the data, and is computationally prohibitive in large-scale applications. To address these issues, we propose a novel measure of information, termed low-rank matrix-based R\'enyi's entropy, based on low-rank representations of infinitely divisible kernel matrices. The proposed entropy functional inherits the specialty of of the original definition to directly quantify information from data, but enjoys additional advantages including robustness and effective calculation. Specifically, our low-rank variant is more sensitive to informative perturbations induced by changes in underlying distributions, while being insensitive to uninformative ones caused by noises. Moreover, low-rank R\'enyi's entropy can be efficiently approximated by random projection and Lanczos iteration techniques, reducing the overall complexity from $\mathcal{O}(n^3)$ to $\mathcal{O}(n^2 s)$ or even $\mathcal{O}(ns^2)$, where $n$ is the number of data samples and $s \ll n$. We conduct large-scale experiments to evaluate the effectiveness of this new information measure, demonstrating superior results compared to matrix-based R\'enyi's entropy in terms of both performance and computational efficiency. 
\end{abstract}

\section{Introduction}\label{sec:intro}
The practical applications of traditional entropy measures e.g. Shannon's entropy \cite{shannon1948mathematical} and R\'enyi's entropy \cite{renyi1961measures} have long been hindered by their heavy reliance on the underlying data distributions, which are extremely hard to estimate or even intractable in high-dimensional spaces \cite{fan2006statistical}. Alternatively, the matrix-based R\'enyi's entropy proposed by \citep{giraldo2014measures} treats the entire eigenspectrum of a normalized kernel matrix as a probability distribution, thus allows direct quantification from given data samples by projecting them in reproducing kernel Hilbert spaces (RKHS) without the exhausting density estimation. This intriguing property makes matrix-based R\'enyi's entropy and its multivariate extensions \cite{yu2019multivariate} successfully applied in various data science applications, ranging from classical dimensionality reduction and feature selection \cite{brockmeier2017quantifying, alvarez2017kernel} problems to advanced deep learning problems such as network pruning \cite{sarvani2021hrel} and knowledge distillation \cite{miles2021information}.

Despite the empirical success of matrix-based R\'enyi's entropy, it has been shown to be not robust against noises in the data \cite{yu2019multivariate}, because it cannot distinguish them from linear combinations of informative features in high-dimensional scenarios. Moreover, the exact calculation requires $\mathcal{O}(n^3)$ time complexity with traditional eigenvalue decomposition techniques e.g. CUR decomposition and QR factorization \cite{mahoney2009cur, watkins2008qr}, greatly hampering its application in large scale tasks due to the unacceptable computational cost.

Inspired by the success of min-entropy which uses the largest outcome solely as a measure of information \cite{wan2018min, konig2009operational}, we seek for a robust information quantity by utilizing low-rank representations of kernel matrices. Our new definition, termed low-rank matrix-based R\'enyi's entropy (abbreviated as low-rank R\'enyi's entropy), fulfills the entire set of axioms provided by R\'enyi \cite{renyi1961measures} that a function must satisfy to be considered a measure of information. Compared to the original matrix-based R\'enyi's entropy, our low-rank variant is more sensitive to informative perturbations caused by variation of the underlying probability distribution, while being more robust to uninformative ones caused by noises in the data samples. Moreover, our low-rank R\'enyi's entropy can be efficiently approximated by random projection and Lanczos iteration techniques, achieving substantially lower time complexity than the trivial eigenvalue decomposition approach. We theoretically analyze the quality of approximation results, and conduct large-scale experiments to evaluate the effectiveness of low-rank R\'enyi's entropy as well as the approximation algorithms. The main contributions of this work are summarized as follows:
\begin{itemize}
    \item We extend Giraldo et al.'s definition and show that a measure of entropy can be built upon the low-rank representation of the kernel matrix. Our low-rank definition can be naturally extended to measure the interactions between multiple random variables, including joint entropy, conditional entropy, and mutual information.
    
    \item Theoretically, we show that low-rank R\'enyi's entropy is more insensitive to random perturbations of the data samples under mild assumptions. We also give empirical examples of low-rank R\'enyi's entropy achieving higher discriminability for different eigenspectrum distributions through a proper choice of the hyper-parameter $k$.
    
    \item We develop efficient algorithms to approximate low-rank R\'enyi's entropy through random projection and Lanczos iteration techniques, enabling fast and accurate estimations respectively. The overall complexity is reduced from $\mathcal{O}(n^3)$ to $\mathcal{O}(n^2s)$ or even $\mathcal{O}(ns^2)$ for some $s \ll n$, leading to a significant speedup compared to the original matrix-based R\'enyi's entropy.
	
	\item We evaluate the effectiveness of low-rank R\'enyi's entropy on large-scale synthetic and real-world datasets, demonstrating superior performance compared to the original matrix-based R\'enyi's entropy while bringing tremendous improvements in computational efficiency.
\end{itemize}

\section{Related Work}
\subsection{Matrix-based R\'enyi's Entropy}
Given random variable $\X$ with probability density function (PDF) $p(\x)$ defined in a finite set $\mathcal{X}$, the $\alpha$-order R\'enyi's entropy ($\alpha > 0, \alpha \neq 1$) $\H_\alpha(\X)$ is defined as 
\begin{equation*}\textstyle
	\H_{\alpha}(\X) = \frac{1}{1 - \alpha} \log_2 \int_{\mathcal{X}} p^\alpha(\x) \dif \x,
\end{equation*}
where the limit case $\alpha \rightarrow 1$ yields the well-known Shannon's entropy.
It is easy to see that R\'enyi's entropy relies heavily on the distribution of the underlying variable $\X$, preventing its further adoption in data-driven science, especially for high-dimensional scenarios. To alleviate this issue, an alternative measure namely matrix-based R\'enyi's entropy was proposed \cite{giraldo2014measures}:
\begin{definition} \label{def_renyi}
	Let $\kappa: \mathcal{X} \times \mathcal{X} \mapsto \mathbb{R}$ be an infinitely divisible positive kernel \cite{bhatia2006infinitely}. Given $ \{\x_i \}_{i=1}^n \subset \mathcal{X}$, each $\x_i$ being a real-valued scalar or vector, and the Gram matrix $K$ obtained from $K_{ij} = \kappa(\x_i, \x_j)$, a matrix-based analogue to R\'enyi's $\alpha$-entropy can be defined as:
	\begin{equation*}\textstyle
		\S_\alpha(\A) = \frac{1}{1-\alpha}\log_2 \Big(\sum_{i=1}^n \lambda_i^\alpha(\A) \Big),
	\end{equation*}
	where $\A_{ij} = \frac{1}{n}\frac{K_{ij}}{\sqrt{K_{ii}K_{jj}}}$ is a normalized kernel matrix and $\lambda_i(\A)$ is the $i$-th largest eigenvalue of $\A$.
\end{definition}
The kernel matrix $\A$ is positive semi-definite (PSD) and satisfies $\tr(\A) = 1$, therefore $\lambda_i \in [0, 1]$ for all $i \in [1,n]$. With this setting, one can similarly define matrix notion of R\'enyi's conditional entropy $\S_\alpha(\A|\B)$, mutual information $\I_\alpha(\A;\B)$, and their multivariate extensions \cite{yu2019multivariate}.   

\subsection{Approximating Matrix-based R\'enyi's Entropy}
Exactly calculating $\S_\alpha(\A)$ requires $\mathcal{O}(n^3)$ time complexity in general with traditional eigenvalue decomposition techniques. Recently, several attempts have been made towards accelerating the computation of $\S_\alpha(\A)$ from the perspective of randomized numerical linear algebra \cite{gong2021computationally, dong2022optimal}. Although we also develop fast approximations, the motivation and technical solutions are totally different: we aim to propose a new measure of information that is robust to noise in data and also enjoys fast computation, whereas Gong and Dong et al. only accelerate the original matrix-based R\'enyi's entropy. Moreover, in terms of adopted mathematical tools, we mainly focus on random projection and Lanczos iteration algorithms, rather than stochastic trace estimation and polynomial approximation techniques used in their works. As a result, the corresponding theoretical error bounds are also different.

\section{A Low-rank Definition of R\'enyi's Entropy}
Our motivations root in two observations. Recall that the min-entropy \cite{konig2009operational}, defined by $\H_{\mathrm{min}}(\X) = -\log_2 \max_{\x \in \mathcal{X}} p(\x)$, measures the amount of information using solely the largest probability outcome. In terms of quantum statistical mechanics, it is the largest eigenvalue of the quantum state $\rho$ which is PSD and has unit trace \cite{ohya2004quantum}. On the other hand, the eigenvalues with the maximum magnitude characterize the main properties of a PSD matrix. Inspired by these observations, we develop a robust information theoretical quantity by exploiting the low-rank representation:
\begin{definition} \label{def_lowrank}
	Let $\kappa: \mathcal{X} \times \mathcal{X} \mapsto \mathbb{R}$ be an infinitely divisible kernel. Given $\{\x_i \}_{i=1}^n \subset \mathcal{X}$ and integer $k \in [1, n-1]$, the low-rank R\'enyi's $\alpha$-order entropy is defined as:
	\begin{equation*}\textstyle
		\S_\alpha^k(\A) = \frac{1}{1-\alpha}\log_2 \prn*{\sum_{i=1}^k \lambda_i^\alpha(\A) + (n-k)\lambda_r^\alpha(\A)},
	\end{equation*}
	where $\A$ is the normalized kernel matrix constructed from $\{\x_i \}_{i=1}^n$ and $\kappa$, $\lambda_i(\A)$ is the $i$-th largest eigenvalues of $\A$ and $\lambda_r(\A) = \frac{1}{n-k}\big(1 - \sum_{i=1}^k \lambda_i(\A) \big)$.
\end{definition}
Let $\A_k$ be the best rank-$k$ approximation of $\A$ and $L_k(\A)$ be the matrix constructed by replacing the smaller $n - k$ eigenvalues in $\A$ to $\lambda_r(\A)$. It is easy to verify that $\S_\alpha^k(\A) = \S_\alpha^k(\A_k) = \S_\alpha^k(L_k(\A)) = \S_\alpha(L_k(\A))$. Definition \ref{def_lowrank} complements the smaller eigenvalues through a uniform distribution, which is the \textbf{unique method} that fulfills all axioms below (the uniqueness is discussed in the appendix\footnote{\url{https://github.com/Gamepiaynmo/LRMI}}).
\begin{proposition} \label{prop_lowrank}
    Let $\A, \B \in \mathbb{R}^{n \times n}$ be arbitrary normalized kernel matrices, then
    \begin{enumerate}[(a)]
        \item $\S_\alpha^k(\P\A\P^\top) = \S_\alpha^k(\A)$ for any orthogonal matrix $\P$.
        \item $\S_\alpha^k(p\A)$ is a continuous function for $0 < p \le 1$.
        \item $0 \le \S_\alpha^k(\A) \le \S_\alpha^k(\frac{1}{n}\I) = \log_2(n)$.
        \item $\S_\alpha^{2nk-k^2}\big(L_k(\A) \otimes L_k(\B)\big) = \S_\alpha^k(\A) + \S_\alpha^k(\B)$.
        \item If $\A\B = \B\A = \0$ and $\tr(\A_k) = \tr(\B_k) = 1$, then for $g(x) = 2^{(1-\alpha)x}$ and $t \in [0, 1]$, we have $\S_\alpha^{2k}\big(t\A + (1-t)\B\big) = g^{-1}\big( tg(\S_\alpha^k(\A)) + (1-t)g(\S_\alpha^k(\B)) \big)$.
        \item $\S_\alpha^k\prn*{\frac{\A\circ\B}{\tr(\A\circ\B)}} \ge \max\prn*{\S_\alpha^k(\A), \S_\alpha^k(\B)}$.
        \item $\S_\alpha^k\prn*{\frac{\A\circ\B}{\tr(\A\circ\B)}} \le \S_\alpha^k(\A) + \S_\alpha^k(\B)$.
    \end{enumerate}
\end{proposition}

\begin{remark}
    Proposition \ref{prop_lowrank} characterizes the basic properties of low-rank R\'enyi's entropy, in which (a)-(e) are the set of axioms provided by R\'enyi \cite{renyi1961measures} that a function must satisfy to be a measure of information. Additionally, (f) and (g) together imply a definition of joint entropy which is also compatible with the individual entropy measures:
    \begin{equation*}\textstyle
        \S_\alpha^k(\A, \B) = \S_\alpha^k\prn*{\frac{\A\circ\B}{\tr(\A\circ\B)}}.
    \end{equation*}
    This further allows us to define the low-rank conditional entropy $\S_\alpha^k(\A|\B)$ and mutual information $\I_\alpha^k(\A;\B)$, whose positiveness is guaranteed by (f) and (g) respectively:
    \begin{gather*}
        \S_\alpha^k(\A|\B) = \S_\alpha^k(\A, \B) - \S_\alpha^k(\B), \\
        \I_\alpha^k(\A;\B) = \S_\alpha^k(\A) + \S_\alpha^k(\B) - \S_\alpha^k(\A, \B).
    \end{gather*}
\end{remark}

An intuitive overview of the comparative behavior between $\S_\alpha(\A)$ and $\S_\alpha^k(\A)$ for $n = 1000$ is reported in Figure \ref{fig_intuitive} and \ref{fig_intuitive1}, where we evaluate the impact of $k$, $\alpha$ and eigenspectrum decay rate (EDR) $r$ respectively. The eigenvalues are initialized by $\lambda_i = e^{-ri/n}$ and then normalized. It can be observed from Figure \ref{fig_intuitive} that $\S_\alpha^k(\A)$ is always larger than $\S_\alpha(\A)$ since the uncertainty of the latter $n - k$ outcomes are maximized. Moreover, $\S_\alpha^k(\A)$ quickly converges to $\S_\alpha(\A)$ with the increase of $k$, especially in extreme cases when the eigenspectrum of $\A$ is flat or steep. From Figure \ref{fig_intuitive1}, we can see that for small $k$, $\S_\alpha^k(\A)$ decreases slow with the increase of $\alpha$ when $\alpha < 1$ and fast otherwise. This behavior is the opposite when $k$ becomes large. Furthermore, we can see that EDR directly influences the value of entropy, as a flat eigenspectrum indicates higher uncertainty and steep the opposite. As can be seen, $\S_\alpha^k(\A)$ monotonically decreases with the increase of $r$, and decreases faster than $\S_\alpha(\A)$ in a certain range which varies according to the choice of $k$, indicating higher sensitivity to informative distribution changes when the hyper-parameter $k$ is selected properly.

\begin{figure}[t]
	\small
	\centering
	\includegraphics[width=0.45\textwidth]{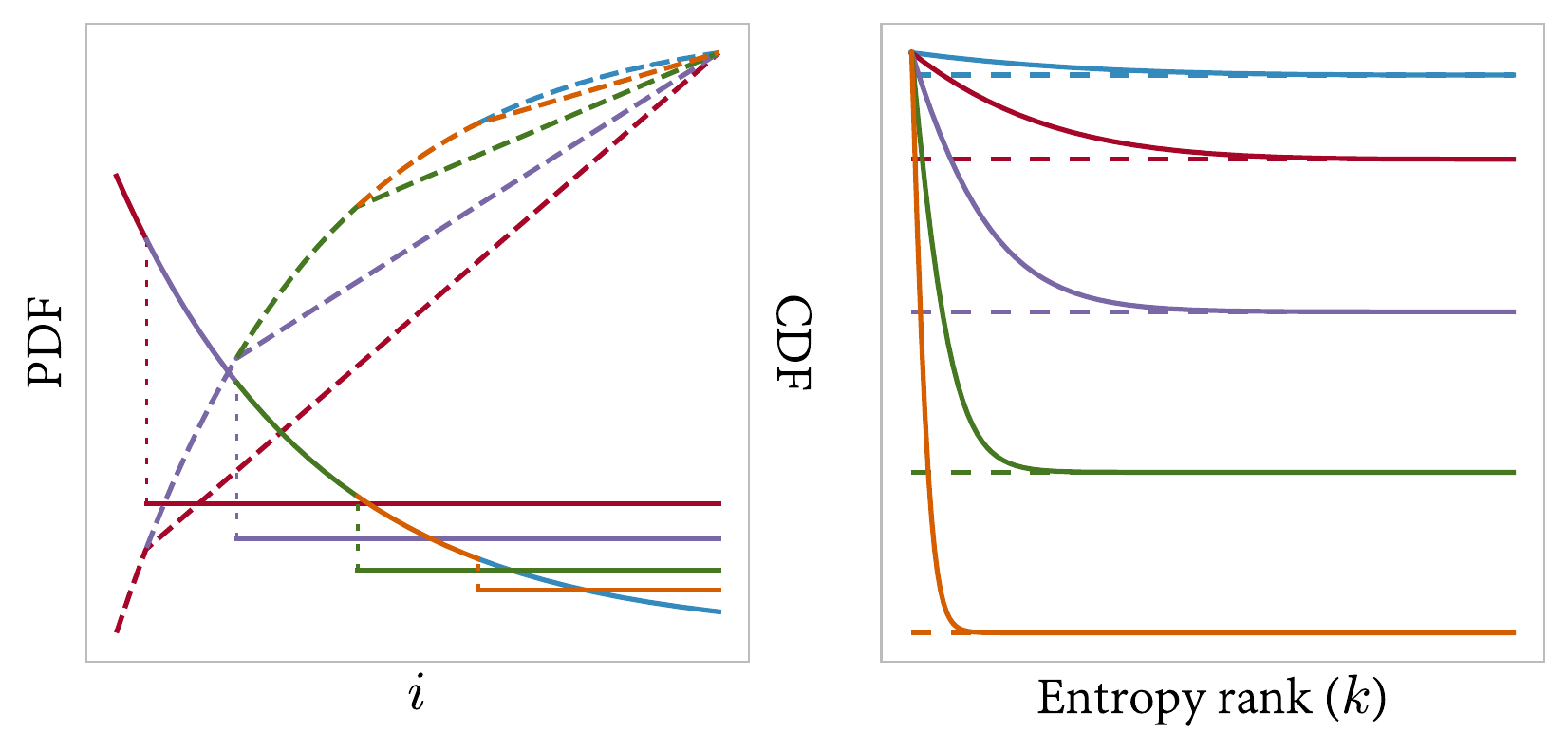}
	\caption{Left: PDF (solid) and CDF (dashed) of the altered eigenspectrum for different ranks $k$. Right: The convergence behavior of $\S_\alpha^k(\A)$ (solid) to $\S_\alpha(\A)$ (dashed) with the increase of rank $k$ for different EDR ($r$).}
	\label{fig_intuitive}
\end{figure}

Moreover, consider the case that the data samples $\{\x_i\}_{i=1}^n$ are randomly perturbed, i.e. $\y_i = \x_i + \varepsilon\p_i$, where $\p_i$ are random vectors comprised of i.i.d. entries with zero expectation and unit variance. Let $\A$ and $\B$ be kernel matrices constructed from $\{\x_i\}_{i=1}^n$ and $\{\y_i\}_{i=1}^n$ respectively, and let $\{\lambda_i\}_{i=1}^n$, $\{\mu_i\}_{i=1}^n$ be their eigenvalues. Then it satisfies that $\mu_i \approx \lambda_i + \u_i^\top(\B-\A)\u_i$ \cite{ngo2005approach}, where $\u_i$ is the corresponding eigenvector of $\lambda_i$. When $\varepsilon$ is small, the entries as well as the eigenvalues of $\A$ are nearly independently perturbed. The following theorem shows that $\S_\alpha^k(\A)$ is more robust against small noises in data compared to $\S_\alpha(\A)$:
\begin{theorem} \label{th:robust}
    Let $\{\nu_i\}_{i=1}^n$ be independent random variables with zero mean and variance $\{\sigma_i^2\}_{i=1}^n$. Let $\A$ and $\B$ be PSD matrices with eigenvalues $\lambda_i$ and $\mu_i = \lambda_i + \nu_i$ respectively. If $\sum_{i=1}^k\sigma_i^2 \le \sum_{i=k+1}^n\sigma_i^2$ or $\alpha > 1$, there exists $\epsilon > 0$ such that when all $|\nu_i| \le \epsilon$, we have $\Var[\IP_\alpha^k(\B)] \le \Var[\IP_\alpha(\B)]$, where $\IP$ is the information potential \cite{gokcay2000new} defined as $\IP_\alpha(\B) = 2^{(1-\alpha)\S_\alpha(\B)}$ and $\IP_\alpha^k(\B) = 2^{(1-\alpha)\S_\alpha^k(\B)}$.
\end{theorem}
\begin{remark}
    Theorem \ref{th:robust} indicates that $\IP_\alpha^k(\B)$ enables lower variance than $\IP_\alpha(\B)$ against random perturbation of the eigenvalues under mild conditions, which is easy to be satisfied since in most cases we have $k \ll n$. Combining with our discussion above, the low-rank R\'enyi's entropy is more sensitive to informative variations in probability distributions which will surely induce an increase or decrease in entropy, while being insensitive to uninformative perturbations caused by noises in the data samples.
\end{remark}

\subsection{Extending to Multivariate Scenarios}
Following Definition \ref{def_lowrank} and Proposition \ref{prop_lowrank}, the low-rank variant of multivariate R\'enyi's joint entropy, in virtue of the Venn diagram relation for Shannon's entropy \cite{yeung1991new}, could be naturally derived:
\begin{definition} \label{def_lowrank_joint}
	Let $\{\kappa_i\}_{i=1}^L: \mathcal{X}^i \times \mathcal{X}^i \mapsto \mathbb{R}$ be positive infinitely divisible kernels and $\{\x_i^1, \cdots, \x_i^L\}_{i=1}^n \subset \mathcal{X}^1 \times \cdots \times \mathcal{X}^L$, the low-rank R\'enyi's joint entropy is defined as:
	\begin{equation*}\textstyle
		\S_\alpha^k(\A_1, \cdots, \A_L) = \S_\alpha^k \prn*{\frac{\A_1 \circ \cdots \circ \A_L}{\mathbf{tr}(\A_1 \circ \cdots \circ \A_L)}},
	\end{equation*}
    where $\A_1, \cdots, \A_L$ are normalized kernel matrices constructed from $\{\x_i^1\}_{i=1}^n$, $\cdots$, $\{\x_i^L\}_{i=1}^n$ respectively and $\circ$ denotes the Hadamard product.
\end{definition}

\begin{figure}[t]
	\small
	\centering
	\includegraphics[width=0.45\textwidth]{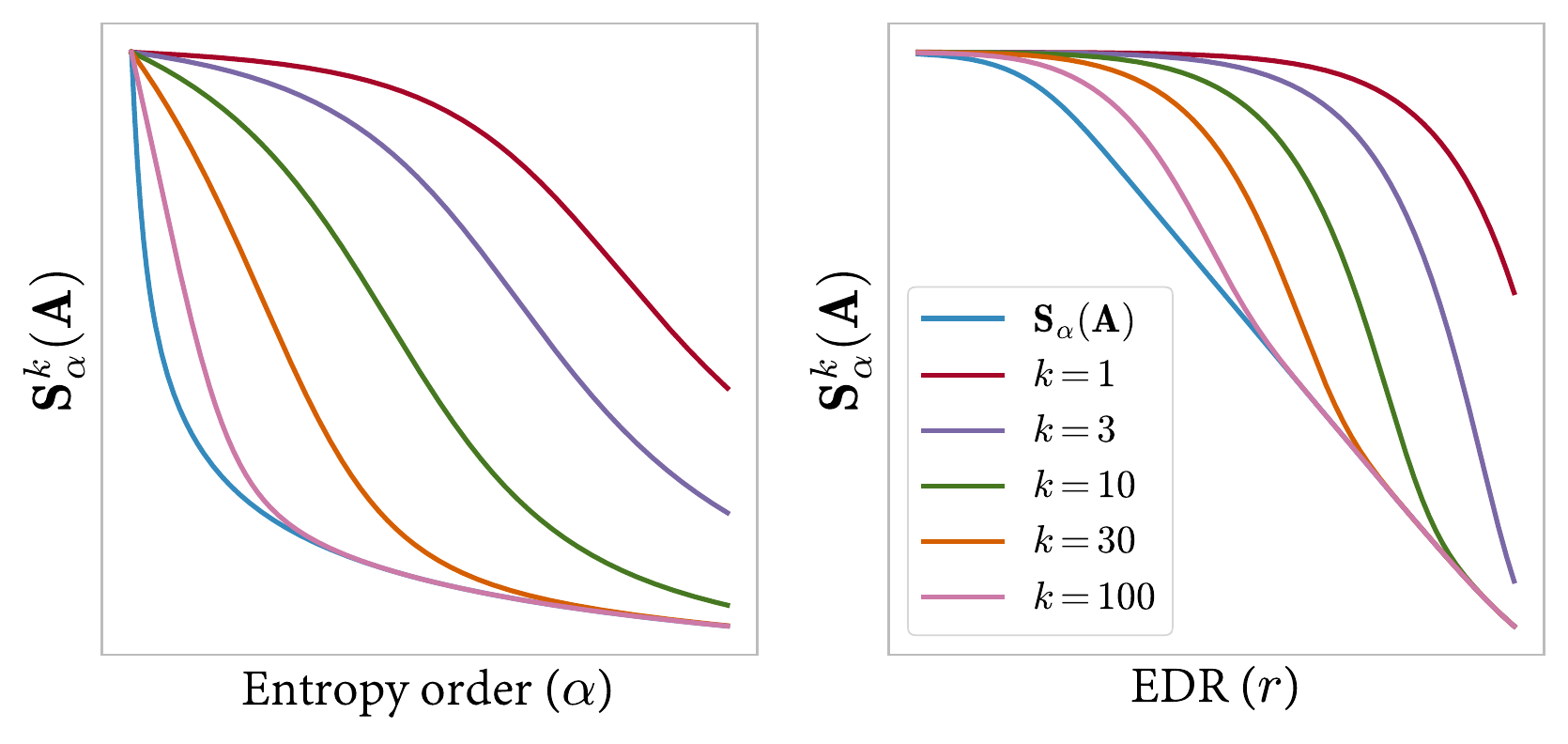}
	\caption{Left: The behavior of $\S_\alpha^k(\A)$ when the entropy order $\alpha$ varies from $0$ to $2$. Right: The behavior of $\S_\alpha^k(\A)$ when the EDR of $\A$ varies from flat to steep.}
	\label{fig_intuitive1}
\end{figure}

This joint entropy definition enables further extension to multivariate conditional entropy and mutual information:
\begin{gather*}
    \S_\alpha^k(\A_1, \cdots, \A_k| \B) = \S_\alpha^k(\A_1, \cdots, \A_k, \B) - \S_\alpha^k(\B), \\
    \begin{aligned}
    	\I_\alpha^k(\A_1, \cdots, \A_k; \B) =&\, \S_\alpha^k(\A_1, \cdots, \A_k) + \S_\alpha^k(\B) \\
    	&- \S_\alpha^k(\A_1, \cdots, \A_k, \B),
    \end{aligned}
\end{gather*}
where $\A_1, \cdots, \A_L$ and $\B$ are normalized kernel matrices constructed from the variables $\{\x_i^1\}_{i=1}^n$, $\cdots$, $\{\x_i^L\}_{i=1}^n$ and the target label $\{\y_i\}_{i=1}^n$ respectively. Their positiveness can be guaranteed through a reduction to axiom (f) and (g). These multivariate information quantities enable much more widespread applications e.g. feature selection, dimension reduction and information-based clustering.

\section{Approximating Low-rank R\'enyi's Entropy}
Although only the largest eigenvalues are accessed by our entropy definition, one still needs to calculate the full eigenspectrum of the PSD matrix $\A$ through eigenvalue decomposition algorithms, resulting in $\mathcal{O}(n^3)$ overall time cost. To alleviate the computational burden, we design fast approximations by leveraging random projection and Lanczos iteration techniques for low-rank R\'enyi's entropy. 

\subsection{Random Projection Approach}
Random projection offers a natural way to approximate the low-rank representation of kernel matrices. The core idea is to project the $n \times n$ PSD matrix $\A$ into a $n \times s$ subspace, and then use the largest $k$ singular value of the projected matrix as approximations of the largest $k$ eigenvalues, as summarized in Algorithm \ref{alg_proj}. In this way, the main computation cost is reduced to $\mathcal{O}(n^2s)$ or even $\mathcal{O}(ns^2)$, ($s \ll n$), substantially lower than the original $\mathcal{O}(n^3)$ approach. Based on this fact, we develop efficient approximation algorithms by exploring different random projection techniques, in which the construction of $\P$ varies depending on the practical applications, ranging from simple but effective Gaussian distributions to advanced random orthogonal projections.

\begin{algorithm}[t]
	\begin{algorithmic}[1]
		\STATE \textbf{Input:} Integers $n$, $k \in [1,n/2], s \ge k$, kernel matrix $\A \in \mathbb{R}^{n \times n}$, order $\alpha > 0$.
		\STATE \textbf{Output:} Approximation to $\S_\alpha^k(\A)$;
		
		\STATE Construct a random projection matrix $\P \in \mathbb{R}^{n \times s}$.
		\STATE Calculate $\hat{\A} =  \A \P \in \mathbb{R}^{n \times s}$.
		\STATE Calculate the largest $k$ singular values $\hat{\lambda}_i$, $i \in [1,k]$ of $\hat{\A}$ through singular value decomposition.
		\STATE Calculate $\hat{\lambda}_r = \frac{1}{n-k}\prn*{1 - \sum_{i=1}^{k} \hat{\lambda}_i}$.
		
		\STATE \textbf{Return:} $\hat{\S}_\alpha^k(\A) = \frac{1}{1-\alpha} \log_2 \prn*{\sum_{i=1}^{k} \hat{\lambda}_i^\alpha + (n-k)\hat{\lambda}_r^\alpha}$.
	\end{algorithmic}
	\caption{Approximation via Random Projection}
	\label{alg_proj}
\end{algorithm}

\subsubsection{Gaussian Random Projection} \leavevmode\\
As one of the most widely used random projection techniques, Gaussian random projection (GRP) admits a simple but elegant solution for eigenvalue approximation:
\begin{equation*}\textstyle
    \P = \sqrt{n/s} \cdot \G,
\end{equation*}
where the columns of $\G \in \mathbb{R}^{n \times s}$ are initialized by i.i.d random standard Gaussian variables and then orthogonalized. The time complexity of GRP is $\mathcal{O}(n^2s)$.

\subsubsection{Subsampled Randomized Hadamard Transform} \leavevmode\\
SRHT \cite{lu2012faster, tropp2011improved} is a simplification of the fast Johnson-Lindenstrauss transform \cite{ailon2009fast} which
preserves the geometry of an entire subspace of vectors compared to GRP. In our settings, the $n \times s$ SRHT matrix is constructed by
\begin{equation*}\textstyle
	\P = \sqrt{1/s} \cdot \D \H \S,
\end{equation*}
where $\D \in \mathbb{R}^{n \times n}$ is a diagonal matrix with random $\{\pm 1 \}$ entries, $\H \in \mathbb{R}^{n \times n}$ is a Walsh-Hadamard matrix, $\S \in \mathbb{R}^{n \times s}$ is a subsampling matrix whose columns are a uniformly chosen subset of the standard basis of $\mathbb{R}^n$.

Two key ingredients make SRHT an efficient approximation strategy: first, it takes only $\mathcal{O}(n^2 \min(\log(n),s))$ time complexity to calculate the projected matrix $\hat{\A}$; second, the orthonormality between the columns of $\A$ can be preserved after projection, thus is more likely to achieve lower approximation error compared to GRP.

\subsubsection{Input-Sparsity Transform} \leavevmode\\
Similar to SRHT, input-sparsity transform (IST) \cite{mahoney2011randomized, woodruff2020near} utilizes the fast John-Lindenstrauss transform to reduce time complexity for least-square regression and low-rank approximation:
\begin{equation*}\textstyle
    \P = \sqrt{n/s} \cdot \D \S,
\end{equation*}
where $\D$ and $\S$ are constructed in the same way as SRHT. The complexity of calculating $\hat{\A}$ using IST is $\mathcal{O}(\textrm{nnz}(\A))$, where nnz denotes the number of non-zero entries, resulting in a total complexity of $\mathcal{O}(\min(\textrm{nnz}(\A), ns^2))$.

\subsubsection{Sparse Graph Sketching} \leavevmode\\
The idea of using sparse graphs as sketching matrices is proposed in \cite{hu2021sparse}. It is shown that the generated bipartite graphs by uniformly adding edges enjoy elegant theoretical properties known as the Expander Graph or Magical Graph with high probability, and thus serve as an effective random projection strategy:
\begin{equation*}\textstyle
    \P = \sqrt{1/p} \cdot \G,
\end{equation*}
where $p \in \mathbb{N}$ is the hyper-parameter that controls the sparsity, and each column $\g$ of $\G$ is constructed independently by uniformly sampling $c \subset [n]$ with $|c| = p$, and then setting $\g_i = \{\pm 1\}$ randomly for $i \in c$ and $\g_i = 0$ for $i \notin c$. Similar to IST, sparse graph sketching (SGS) also utilizes the sparsity of input matrices and achieves $\mathcal{O}(\textrm{nnz}(\A)p)$ computational complexity to calculate the projected matrix.

\subsubsection{Theoretical Results} \leavevmode\\
Next, we provide the main theorem on characterizing the quality-of-approximation for low-rank R\'enyi's entropy:
\begin{theorem} \label{th:rp_approx}
	Let $\A$ be positive definite and
	\begin{gather*}
	    s = \begin{cases}
	        \mathcal{O}(k + \log(1/\delta)/\epsilon_0^2), & \textrm{for GRP} \\
	        \mathcal{O}((k + \log n)\log k/\epsilon_0^2), & \textrm{for SRHT} \\
	        \mathcal{O}(k^2/\epsilon_0^2), & \textrm{for IST} \\
	        \mathcal{O}(k\log(k/\delta\epsilon_0)/\epsilon_0^2), & \textrm{for SGS}
	    \end{cases} \\
	    p = \mathcal{O}(\log(k/\delta\epsilon_0)/\epsilon_0), \qquad\textrm{for SGS}
	\end{gather*}
	where $\epsilon_0 = \epsilon\lambda_k\lambda_r$, then for $k \le n/2$, with confidence at least $1-\delta$, the output of Algorithm \ref{alg_proj} satisfies
	\begin{equation*}
		|\lambda_i^2 - \hat{\lambda}_i^2 | \leq \epsilon
	\end{equation*} 
	for all $i \in [1,k]$ eigenvalues of $\A$ and
	\begin{equation*}\textstyle
		|\S_\alpha^k(\A) - \hat{\S}_\alpha^k(\A)| \leq \abs{\frac{\alpha}{1-\alpha} \log_2 \prn*{1-\epsilon}}.
	\end{equation*}
\end{theorem}

\begin{remark}
	Theorem \ref{th:rp_approx} provides the accuracy guarantees for low-rank R\'enyi's entropy approximation via random projections. It can be observed that the approximation error grows with the increase of $\alpha$ when $\alpha$ is small. Note that although the error bound is additive in nature, it can be further reduced to a relative error bound under mild condition $\S_\alpha^k(G) \geq \sqrt{\epsilon}$. In general, Theorem \ref{th:rp_approx} requires $s = \mathcal{O}(k + 1/\epsilon^2)$ to achieve $1\pm\epsilon$ absolute accuracy, which is consistent with the complexity results of least squares and low rank approximations \cite{mahoney2011randomized}.
\end{remark}

\subsection{Lanczos Iteration Approach}
Besides random projection, the Lanczos algorithm is also widely adopted to find the $k$ extreme (largest or smallest in magnitude) eigenvalues and the corresponding eigenvectors of an $n \times n$ Hermitian matrix $\A$. Given an initial vector $\q$, the Lanczos algorithm utilizes the Krylov subspace spanned by $\{\q, \A\q, \cdots, \A^s\q\}$ to construct an tridiagonalization of $\A$ whose eigenvalues converge to those of $\A$ along with the increase of $s$, and are satisfactorily accurate even for $s \ll n$. As shown in Algorithm \ref{alg_lanczos}, the main computation cost is the $\mathcal{O}(n^2s)$ matrix-vector multiplications in the Lanczos process, which could be further reduced to $\mathcal{O}(\mathrm{nnz}(\A)s)$ when $\A$ is sparse. The computational cost of reorthogonalization can be further alleviated by explicit or implicit restarting Lanczos methods. The following theorem establishes the accuracy guarantee of Algorithm \ref{alg_lanczos}:

\begin{algorithm}[t]
    \caption{Approximation via Lanczos Iteration}
    \label{alg_lanczos}
    \begin{algorithmic}[1]
    	\STATE \textbf{Input:} Integers $n$, $k \in [1,n/2], s \ge k$, kernel matrix $\A \in \mathbb{R}^{n \times n}$, order $\alpha > 0$, initial vector $\q$.
    	\STATE \textbf{Output:} Approximation to $\S_\alpha^k(G)$.
    	
    	\STATE Set $\q_0 = 0, \beta_0 = 0, \q_1 = \q / \norm{\q}$.
    	\FOR{$j = 1, 2, \cdots, s$}
    	\STATE $\hat{\q}_{j+1} = \A\q_j - \beta_{j-1}\q_{j-1}$, $\gamma_j = \ang{\hat{\q}_{j+1}, \q_j}$.
    	\STATE $\hat{\q}_{j+1} = \hat{\q}_{j+1} - \gamma_j\q_j$.
    	\STATE Orthogonalize $\hat{\q}_{j+1}$ against $\q_1, \cdots, \q_{j-1}$.
    	\STATE $\beta_j = \norm{\hat{\q}_{j+1}}$, $\q_{j+1} = \hat{\q}_{j+1} / \beta_j$.
    	\ENDFOR
    	\STATE Calculate the largest $k$ eigenvalues $\hat{\lambda}_i$, $i \in [1, k]$ of \\
    	$\qquad \T = \begin{bmatrix}
            \gamma_1 & \beta_1 & & 0 \\
            \beta_1 & \gamma_2 & & \\
            & & \ddots & \beta_{s-1} \\
            0 & & \beta_{s-1} & \gamma_s \\
        \end{bmatrix}$.
		\STATE Calculate $\hat{\lambda}_r = \frac{1}{n-k}\prn*{1 - \sum_{i=1}^{k} \hat{\lambda}_i}$.
		
		\STATE \textbf{Return:} $\hat{\S}_\alpha^k(\A) = \frac{1}{1-\alpha} \log_2 \prn*{\sum_{i=1}^{k} \hat{\lambda}_i^\alpha + (n-k)\hat{\lambda}_r^\alpha}$.
    \end{algorithmic}
\end{algorithm}

\begin{theorem} \label{th:lanczos_approx}
    Let $\A$ be positive definite, $\q$ be the initial vector, $\{\pphi_i\}_{i=1}^k$ be the corresponding eigenvectors and
    \begin{equation*}
        \textstyle
        s = \ceil*{k + \frac{1}{2\log R} \log\prn*{\frac{4\theta^2K^2\lambda_1}{\epsilon\lambda_r}}},
    \end{equation*}
    where
    \begin{gather*}
        \textstyle
        R = \gamma + \sqrt{\gamma^2 - 1}, \quad \gamma = 1 + 2\min_{i\in[1,k]} \frac{\lambda_i - \lambda_{i+1}}{\lambda_{i+1} - \lambda_n}, \\
        \textstyle
        \theta = \max_{i\in[1,k]} \tan\ang{\pphi_i, \q}, \quad K = \prod_{j=1}^{k-1} \frac{\hat{\lambda}_j - \lambda_n}{\hat{\lambda}_j - \lambda_k},
    \end{gather*}
    then for $k \le n/2$, the output of Algorithm \ref{alg_lanczos} satisfies\begin{equation*}
        \textstyle
		0 \le \lambda_i - \hat{\lambda}_i \leq \epsilon \lambda_i
	\end{equation*} 
	for all $i \in [1,k]$ eigenvalues of $\A$ and
	\begin{equation*}\textstyle
		|\S_\alpha^k(\A) - \hat{\S}_\alpha^k(\A)| \leq \abs{\frac{\alpha}{1-\alpha} \log_2 \prn*{1-\epsilon}}.
	\end{equation*}
\end{theorem}

\begin{remark}
    Theorem \ref{th:lanczos_approx} provides the accuracy guarantee for the Lanczos algorithm. The relationship between approximation error and $\alpha$ is similar to those in Theorem \ref{th:rp_approx}. Algorithm \ref{alg_lanczos} achieves a much faster convergence rate compared to Algorithm \ref{alg_proj} while achieving the same level of absolute precision. When $\epsilon$ is small, $R$, $\theta$ and $K$ can be regarded as constants that depends only on the eigenspectrum of $\A$ and the initial vector $\q$, so that $s = \mathcal{O}(k + \log(1/\epsilon))$ is enough to guarantee a $1\pm\epsilon$ accuracy. In practice, $\q$ is suggested to be generated by random Gaussian in order to avoid a large $\theta$ with high probability \cite{urschel2021uniform}.
\end{remark}

\section{Experimental Results}
In this section, we evaluate the proposed low-rank R\'enyi's entropy and the approximation algorithms under large-scale experiments. Our experiments are conducted on an Intel i7-10700 (2.90GHz) CPU and an RTX 2080Ti GPU with 64GB of RAM. The algorithms are implemented in C++ with the Eigen library and in Python with the Pytorch library.

\subsection{Simulation studies}
We first test the robustness of $\S_\alpha^k(\A)$ against noises in the data. As indicated by Theorem \ref{th:robust}, low-rank R\'enyi's entropy achieves lower variance under mild conditions in terms of the information potential. We consider the case that the input data points are randomly perturbed, i.e. $\y_i = \x_i + \varepsilon\p_i$ for $i \in [1, n]$, where $\p_i$ is comprised of i.i.d. random variables. Let $\{\lambda_i\}_{i=1}^n$, $\{\mu_i\}_{i=1}^n$ denote the eigenvalues of normalized kernel matrices constructed from $\{\x_i\}_{i=1}^n$ and $\{\y_i\}_{i=1}^n$ respectively. We test the following noise distributions: Standard Gaussian $N(0,1)$, Uniform $U(-\sqrt{3}, \sqrt{3})$, Student-t $t(3)/\sqrt{3}$ and Rademacher $\{\pm 1\}$ with $n = 100$ (detailed settings are given in the appendix). The examples of variation in eigenvalues ($\mu_i - \lambda_i$) and the standard deviation (multiplied by $n$) of entropy values after $100$ trials are reported in Figure \ref{fig_variance}. It verifies our analysis that when $\varepsilon$ is small, the eigenvalues $\mu_i$ are nearly independently perturbed. Moreover, our low-rank definition achieves lower variance than matrix-based R\'enyi's entropy under different choices of $\alpha$, in which smaller $k$ corresponds to higher robustness.

\begin{figure*}[t]
	\small
	\centering
	\includegraphics[width=0.9\textwidth]{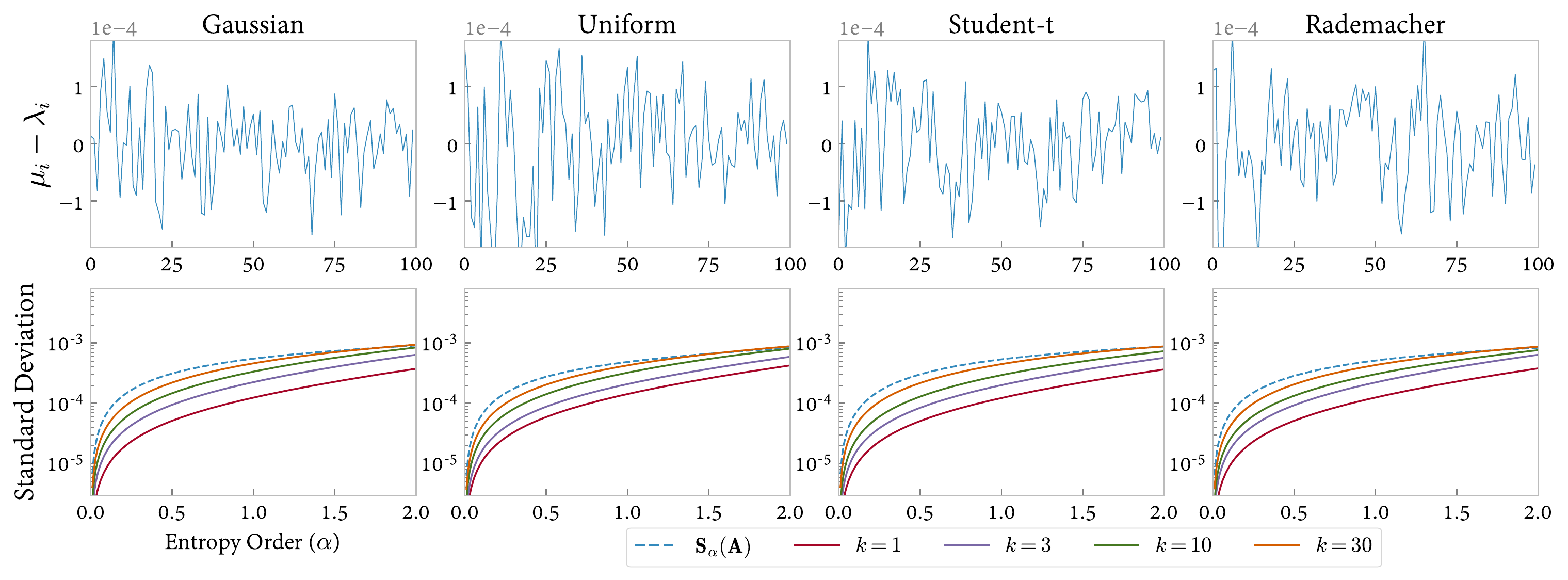}
	\caption{Upper: perturbation of the eigenvalues, i.e. $\mu_i - \lambda_i$. Lower: standard deviation of matrix-based R\'enyi's entropy and low-rank R\'enyi's entropy against random perturbations of the data samples for different values of $\alpha$.}
	\label{fig_variance}
\end{figure*}

\subsection{Real Data Examples}
In this section, we demonstrate the great potential of applying our low-rank R\'enyi's entropy functional and its multivariate extensions in two representative real-world information-related applications, which utilize the mutual information (information bottleneck) and multivariate mutual information (feature selection) respectively.

\begin{table}[t]
    \centering
    \small
	\setlength{\tabcolsep}{0.7em}
    \begin{tabular}{ccc}
        \toprule
        Objective & Accuracy (\%) & Training Time (minutes) \\
        \midrule
        CE & 92.64 $\pm$ 0.03 & \hphantom{0}- / \hphantom{0}80 \\
        VIB & 94.08 $\pm$ 0.02 & \hphantom{0}4 / \hphantom{0}84 \\
        NIB & 94.01 $\pm$ 0.04 & \hphantom{0}7 / \hphantom{0}87 \\
        \midrule
        MRIB & \ul{94.13 $\pm$ 0.04} & 46 / 126 \\
        LRIB & \bf{94.16 $\pm$ 0.09} & 15 / \hphantom{0}95 \\
		\bottomrule
    \end{tabular}
    \caption{Classification accuracy and training time of different IB objectives. Left is the time spent on IB calculation and right is the total training time.}
    \label{tbl_ib}
\end{table}

\begin{table*}[t]
	\centering
	\footnotesize
	\setlength{\tabcolsep}{0.27em}
	\begin{tabular}{ ccccccccccc }
		\toprule
		Method & Criterion & Breast & Semeion & Madelon & Krvskp & Spambase & Waveform & Optdigits & Statlog & Average \\
		\midrule
		MIFS & $\scriptstyle \I(\X_{i_l};\Y) - \beta \sum_{j=1}^{l-1} \I(\X_{i_l};\X_{i_j})$ & 4.8 & 2.5 & 6.6 & 3.8 & 7.3 & 6.4 & 4.5 & 3.8 & 4.96 \\
		FOU & $\scriptstyle \I(\X_{i_l};\Y) - \sum_{j=1}^{l-1} [\I(\X_{i_l};\X_{i_j}) - \I(\X_{i_l};\X_{i_j}|\Y)]$ & 5.2 & 2.5 & 5.6 & 1.9 & 6.2 & 5.7 & 4.8 & 5.7 & 4.70 \\
		MRMR & $\scriptstyle \I(\X_{i_l};\Y) - \frac{1}{l-1} \sum_{j=1}^{l-1} \I(\X_{i_l};\X_{i_j})$ & \bf{2.3} & 4.7 & 6.7 & 3.7 & 5.6 & 3.6 & 4.8 & 4.0 & 4.43 \\
		JMI & $\scriptstyle \sum_{j=1}^{l-1} \I(\{\X_{i_l},\X_{i_j}\};\Y)$ & 5.1 & 5.2 & 3.0 & 3.7 & 4.2 & 2.3 & 3.8 & 3.5 & 3.85 \\
		CMIM & $\scriptstyle \min_{j=1}^{l-1} \I(\X_{i_l};\Y|\X_{i_j})$ & 3.0 & 2.7 & 4.5 & 3.6 & 3.2 & 4.7 & 2.6 & 5.6 & 3.74 \\
		DISR & $\scriptstyle \sum_{j=1}^{l-1} \I(\{\X_{i_l},\X_{i_j}\};\Y) / \H(\X_{i_l},\X_{i_j},\Y)$ & 7.3 & 5.7 & 4.0 & 3.2 & 3.9 & 2.3 & 6.9 & 5.8 & 4.89 \\
		\midrule
		MRMI & $\scriptstyle \I_\alpha(\{\X_{i_1},\X_{i_2},\cdots,\X_{i_l}\};\Y)$ & \ul{2.6} & \ul{1.8} & \ul{1.2} & \ul{1.7} & \bf{1.5} & \ul{1.8} & \bf{1.3} & \bf{2.0} & \ul{1.74} \\
		LRMI & $\scriptstyle \I_\alpha^k(\{\X_{i_1},\X_{i_2},\cdots,\X_{i_l}\};\Y)$ & \ul{2.6} & \bf{1.4} & \bf{1.1} & \bf{1.6} & \bf{1.5} & \bf{1.5} & \bf{1.3} & \ul{2.1} & \bf{1.64} \\
		\bottomrule
	\end{tabular}
	\caption{Information theoretic feature selection methods and their average rank over different number of features in each dataset. The first and second best performances are marked as \textbf{bold} and \ul{underlined} respectively.}
	\label{tbl_selection}
\end{table*}

\subsubsection{Application to Information Bottleneck} \leavevmode\\
The Information Bottleneck (IB) methods recently achieve great success in compressing redundant or irrelevant information in the inputs and preventing overfitting in deep neural networks. Formally, given network input $\X$ and target label $\Y$, the IB approach tries to extract a compressed intermediate representation $\Z$ from $\X$ that maintains minimal yet meaningful information to predict the task $\Y$ by optimizing the following IB Lagrangian:
\begin{equation*}
    \mathcal{L}_{\textrm{IB}} = \I(\Y, \Z) - \beta \cdot \I(\X, \Z),
\end{equation*}
where $\beta$ is the hyper-parameter that balances the trade-off between \textbf{sufficiency} (predictive performance of $\Z$ on task $\Y$, quantified by $\I(\Y, \Z)$) and \textbf{minimality} (the complexity of $\Z$, quantified by $\I(\X, \Z)$). In practice, optimizing $\I(\Y, \Z)$ is equivalent to the cross-entropy (CE) loss for classification tasks, so our target remains to optimize the latter term $\I(\X, \Z)$. However, mutual information estimation is extremely hard or even intractable for high-dimension distributions, which is usually the case in deep learning. To address this issue, there have been efforts on using variational approximations to optimize a lower bound of $\I(\X, \Z)$, e.g. Variational IB (VIB) \cite{alemi2017deep} and Nonlinear IB (NIB) \cite{kolchinsky2019nonlinear}. We show that with low-rank R\'enyi's entropy, $\I(\X, \Z)$ can be directly optimized by approximating the largest $k$ eigenvalues of the kernel matrix $\A$ constructed by $\X$ and $\Z$. Recall that the Lanczos method constructs an approximation $\A \approx \Q\T\Q^\top$, where $\Q \in \mathbb{R}^{n \times s}$ has orthogonal columns and $\T \in \mathbb{R}^{s \times s}$ is tridiagonal, we have $\hat{\lambda}_i = \lambda_i(\Q^\top\A\Q)$ for all $i \in [1,s]$. Let $\sum_{i=1}^s\hat{\lambda}_i\u_i\u_i^\top$ be the eigenvalue decomposition of $\Q^\top\A\Q$, we can approximate the gradient of $\S_\alpha^k(\A)$ as:
\begin{equation*}
    \frac{\partial\S_\alpha^k(\A)}{\partial\A} \approx \sum_{i=1}^k \frac{\partial\hat{\S}_\alpha^k(\A)}{\partial\hat{\lambda}_i} \cdot \Q\u_i\u_i^\top\Q^\top.
\end{equation*}
In this experiment, we test the performance of matrix-based R\'enyi's IB (MRIB) \cite{yu2021deep} and our low-rank variant (LRIB) with variational approximation-based objectives using VGG16 as the backbone and CIFAR10 as the classification task. All models are trained for $300$ epochs with $100$ batch size and $0.1$ initial learning rate which is divided by $10$ every $100$ epochs. Following the settings in \cite{yu2021deep}, we select $\alpha = 1.01$, $\beta = 0.01$, $k = 10$ and $s = 20$. The final results are reported in Table \ref{tbl_ib}. It can be seen that the matrix-based approaches MRIB and LRIB outperform other methods, while our LRIB achieves the highest performance with significantly less training time.

\subsubsection{Application to Feature Selection} \leavevmode\\
In practical regression or classification machine learning tasks, many features can be completely irrelevant to the learning target or redundant in the context of others. Given a set of features $\S = \{\X_1, \cdots, \X_L\}$ and the target label $\Y$, we aim to find a subset $\S_{sub} \subset \S$ which leverage the expressiveness and the complexity simultaneously. In the field of information theoretic learning, this target is equivalent to maximizing the multivariate mutual information $\I(\S_{sub};\Y)$, which is computationally prohibitive due to the curse of high dimensionality. As a result, there have been tremendous efforts on approximation techniques that retain only the first or second order interactions and build mutual information estimators upon low-dimensional probability distributions, including Mutual Information-based Feature Selection (MIFS) \cite{battiti1994using}, First-Order Utility (FOU) \cite{brown2009new}, Maximum-Relevance Minimum-Redundancy (MRMR) \cite{peng2005feature}, Joint Mutual Information (JMI) \cite{yang1999data}, Conditional Mutual Information Maximization (CMIM) \cite{fleuret2004fast} and Double Input Symmetrical Relevance (DISR) \cite{meyer2006use} which achieve state-of-the-art performance in information-based feature selection tasks.

We evaluate the performance of matrix-based R\'enyi's mutual information (MRMI) and our low-rank variant (LRMI) with these methods on $8$ widely-used classification datasets as shown in Table \ref{tbl_dataset}, which is chosen to cover a broad variety of instance-feature ratios, number of classes and discreteness. Notice that non-R\'enyi methods can only handle discrete features, so we discretize them into $5$ bins under equal-width criterion as adopted in \cite{brown2012conditional}. In this experiment, we choose the Support Vector Machine (SVM) algorithm with RBF kernel ($\sigma = 1$) as the classifier for continuous datasets and a 3-NN classifier for discrete datasets. Following the settings of \cite{yu2019multivariate}, we select $\alpha \in \{0.6, 1.01, 2\}$, $k \in \{100, 200, 400\}$ via cross-validation, $s = k + 50$ and use the Gaussian kernel of width $\sigma = 1$ for matrix-based entropy measures. Considering that it is NP-hard to evaluate each subset of $\S$, we adopt a greedy strategy to incrementally select $10$ features that maximize our target $\I(\S_{sub};\Y)$. That is, in each step, we fix the current subset $\S_{sub} = \{\X_{i_1}, \cdots, \X_{i_{l-1}}\}$ and add a new feature $\X_{i_l} \in \S / \S_{sub}$ to $\S_{sub}$. The average rank of each method across different number of features and the running time of MRMI and LRMI are reported in Table \ref{tbl_selection} and Table \ref{tbl_dataset}.

\begin{table}[t]
	\centering
	\small
	\setlength{\tabcolsep}{0.45em}
	\begin{tabular}{ cccccr@{\hspace{4pt}/\hspace{4pt}}lc }
		\toprule
		Dataset & \#I & \#F & \#C & Discrete & \multicolumn{2}{c}{Time} & Speedup \\
		\midrule
		Breast & 569 & 30 & 2 & No & 0.31 & 0.25 & 1.2 \\
		Semeion & 1593 & 256 & 10 & Yes & 56 & 44 & 1.3 \\
		Madelon & 2600 & 500 & 2 & Yes & 570 & 39 & 14.4 \\
		Krvskp & 3196 & 36 & 2 & Yes & 71 & 11 & 6.6 \\
		Spambase & 4601 & 56 & 2 & No & 353 & 13 & 27.2 \\
		Waveform & 5000 & 40 & 3 & No & 318 & 14 & 22.5 \\
		Optdigits & 5620 & 64 & 10 & Yes & 750 & 41 & 18.4 \\
		Statlog & 6435 & 36 & 6 & Yes & 600 & 23 & 25.7 \\
		\bottomrule
	\end{tabular}
	\caption{Number of instances (\#I), features (\#F), classes (\#C) and discreteness of classification datasets used in feature selection experiments, running time comparison (minutes) of MRMI (left) and LRMI (right), and speedup ratios.}
	\label{tbl_dataset}
\end{table}

As we can see, both  MRMI and LRMI significantly outperform other Shannon entropy based methods. Compared to MRMI, LRMI achieves $6$ to $27$ times speedup, $15$ times on average via Lanczos approximation. Furthermore, LRMI outperforms MRMI on $4$ datasets in our test benchmark, which verifies our theoretical analysis that low-rank R\'enyi's entropy enables higher robustness against noises in the data. This demonstrates the great potential of our low-rank R\'enyi's entropy on information-related tasks.

\section{Conclusion}
In this paper, we investigate an alternative entropy measure built upon the largest $k$ eigenvalues of the data kernel matrix. Compared to the original matrix-based R\'enyi's entropy, our definition enables higher robustness to noises in the data and sensitivity to informative changes in eigenspectrum distribution with a proper choice of hyper-parameter $k$. Moreover, low-rank R\'enyi's entropy can be efficiently approximated with $\mathcal{O}(ns^2)$ random projection and $\mathcal{O}(n^2s)$ Lanczos iteration techniques, substantially lower than the $\mathcal{O}(n^3)$ complexity required to compute matrix-based R\'enyi's entropy. We conduct large-scale simulation and real-world experiments on information bottleneck and feature selection tasks to validate the effectiveness of low-rank R\'enyi's entropy, demonstrating elegant performance with significant improvements in computational efficiency.

\section{Acknowledgments}
This work was supported by National Key Research and Development Program of China (2021ZD0110700), National Natural Science Foundation of China (62106191, 12071166, 62192781, 61721002), the Research Council of Norway (RCN) under grant 309439, Innovation Research Team of Ministry of Education (IRT\_17R86), Project of China Knowledge Centre for Engineering Science and Technology and Project of Chinese Academy of Engineering (The Online and Offline Mixed Educational Service System for The Belt and Road Training in MOOC China).

\section{Supplementary Experimental Results}

\subsection{Parameter Settings of Robustness Experiment}
In the first simulation study, we set $n = 100$, $d = 400$, $\varepsilon = 1/n = 0.01$ and use the linear kernel $\kappa(\x_i, \x_j) = \x_i^\top\x_j$ to generate the kernel matrices. The data samples $\{\x_i\}_{i=1}^n$ are generated by i.i.d Gaussian distribution $N(0, 0.1^2)$. The noise distributions $\mathcal{E}$ are selected following the criterion that $\E[\mathcal{E}] = 0$ and $\Var[\mathcal{E}] = 1$. It can be seen that we get similar results under different noise settings, which further verifies our analysis that when $\varepsilon$ is small, the variance of $\S_\alpha^k(\A)$ mainly depends on the variance of random perturbations of data samples.

\subsection{Additional Results of Approximation Algorithms}
Additionally, we evaluate the impact of $\alpha$ and $c$ on approximation accuracy. We keep the previous $n = 8192$ parameter settings and set $k = 64$. The results of MRE curves with $c = 1.0$ and varying $\alpha$ are reported in Figure \ref{AlphaExp}. For SGS, we set the sparsity hyper-parameter $p=2$. For Lanczos, we randomly select the initial vector $\q$ from standard Gaussian. It can be seen that GRP achieves the lowest MRE amongst all random projection algorithms. When $\alpha$ is small, all MRE curves exhibit similar behavior and grow with the increase of $\alpha$. This behavior starts to differ when $\alpha$ gets larger. For random projection algorithms, MRE keeps at the same level; for Lanczos algorithm, MRE starts to decrease when $\alpha > 2$. Recall that the larger eigenvalues of $\A$ take the main role in the calculation of $\S_\alpha^k(\A)$ for large $\alpha$, this phenomenon indicates that Lanczos algorithm achieves higher precision for larger eigenvalues than smaller ones (as shown in our proof, it requires only $\mathcal{O}(i+\log(1/\epsilon))$ steps to approximate $\lambda_i$ to relative error $1\pm\epsilon$), while random projection achieves similar level of precision for all of the $k$ eigenvalues.

We then evaluate the impact of EDR (c) on approximation accuracy. In Figure \ref{DecayExp}, we report the MRE curves for $\alpha = 2$ while $c$ varies from $0$ (flat) to $2$ (steep). It is interesting that the behavior of the two types of methods is entirely different. For random projections, the MRE curves grow slowly (GRP) or keep unchanged (SRHT, IST and SGS) when $c$ is small, and start to increase at a constant rate when $c$ gets large. This is because $\S_\alpha^k(\A)$ is decreasing along with the increase of $c$, whose slope is low at first and high when $c$ gets large (see Figure \ref{fig_intuitive1}). This results in the slow to fast increasing behavior in relative error since the absolute error is upper bounded (Theorem \ref{th:rp_approx}). For Lanczos method, the ratio $\lambda_1/\lambda_r$ in Theorem \ref{th:lanczos_approx} increases fast when $c$ is small, resulting in the increase of MRE. When $c$ gets large, $\lambda_1$ gradually reaches its upper bound $\lambda_1 \le 1$, while the intervals between adjacent eigenvalues also increase and result in a higher $R$ (Theorem \ref{th:lanczos_approx}). Moreover, recall that Lanczos algorithm approximates larger eigenvalues to higher precision, these reasons together explain the increase and then decrease MRE of the Lanczos approach.

Next, we conduct large-scale experiments to evaluate the approximation algorithms. The kernel matrices are generated by $\A = \PPhi\SSigma\PPhi^\top$, where $\PPhi \in \mathbb{R}^{n \times n}$ is a random orthogonal matrix, $\SSigma \in \mathbb{R}^{n \times n}$ is a diagonal matrix such that $\SSigma_{ii} = i^{-c}$ for $i \in [1,n]$, and $c$ is a constant that controls the EDR. We set size of the kernel matrix $n = 8192$. For random projection methods, we make $s$ vary from $100$ to $1000$; while for Lanczos algorithm, $s$ varies from $64$ to $110$. The mean relative error (MRE) and $\pm\frac{1}{4}$ standard deviation (SD) are reported in Figure \ref{SExp} for each test after $100$ trials with $\alpha = 1.5$ and $c \in \{1.5, 1.0, 0.5\}$, which correspond to high, medium and low EDR respectively. For comparison, the trivial eigenvalue decomposition approach requires $134$ seconds. It can be seen that all random projection methods yield similar approximation accuracy, in which GRP achieves slightly lower MRE when $c = 0.5$ while IST \& SGS bring the highest speedup. The Lanczos method achieves the highest accuracy with significantly lower $s$ values but requires longer running time. Generally, we recommend IST or SGS for medium precision approximation, and Lanczos when high precision is required. These methods achieve more than $25$ times speedup compared to the trivial approach for an $8192 \times 8192$ kernel matrix.

\subsection{Additional Results of Feature Selection}
The hyper-parameter selection result of $\alpha$ and $k$ for matrix-based R\'enyi's entropy and low-rank R\'enyi's entropy in feature selection experiment via cross-validation are shown in Table \ref{tbl_hyper}. As can be seen, $k = 100$ is already suitable for most circumstances. We perform a Nemenyi's post-hoc test \cite{demvsar2006statistical} to give the significant level, in which the confidence that method $i$ significantly outperforms method $j$ is calculated as:
\begin{equation*}
	p_{ij} = \PPhi\prn*{\left.(R_j - R_i) \middle/ \sqrt{\frac{M(M+1)}{6N}}\right.},
\end{equation*}
where $\PPhi$ is the CDF of standard normal distribution, $R_i$ is the average rank of method $i$, $M$ is the number of methods and $N$ is the number of datasets. For our case, we have $M = 8$, $N = 8$ and the value of $R_i$ are given in the last column of table \ref{tbl_selection}. The confidence level of different methods is shown in Figure \ref{fig_confidence}. It can be seen that under significance level $p = 0.05$, LRMI significantly outperforms all Shannon's entropy-based methods, while the confidence of MRMI outperforming CMIM is not significant enough. In Figure \ref{fig_selection}, we report the classification accuracy achieved by different feature selection methods for the first $10$ features. It can be seen that classification error tends to stabilize after selecting the $10$ most informative features. 

\begin{figure}[t]
	\small
	\centering
	\includegraphics[width=0.4\textwidth]{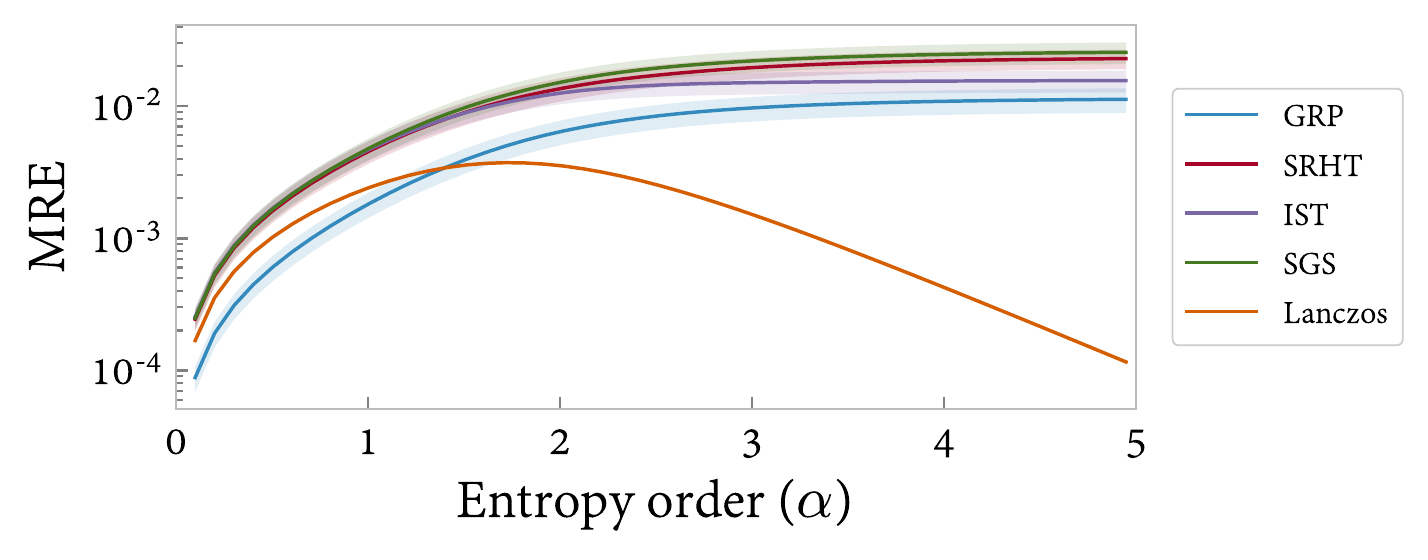}
	\caption{$\alpha$ versus MRE curves for entropy approximation.}
	\label{AlphaExp}
\end{figure}

\begin{figure}[t]
	\small
	\centering
	\includegraphics[width=0.4\textwidth]{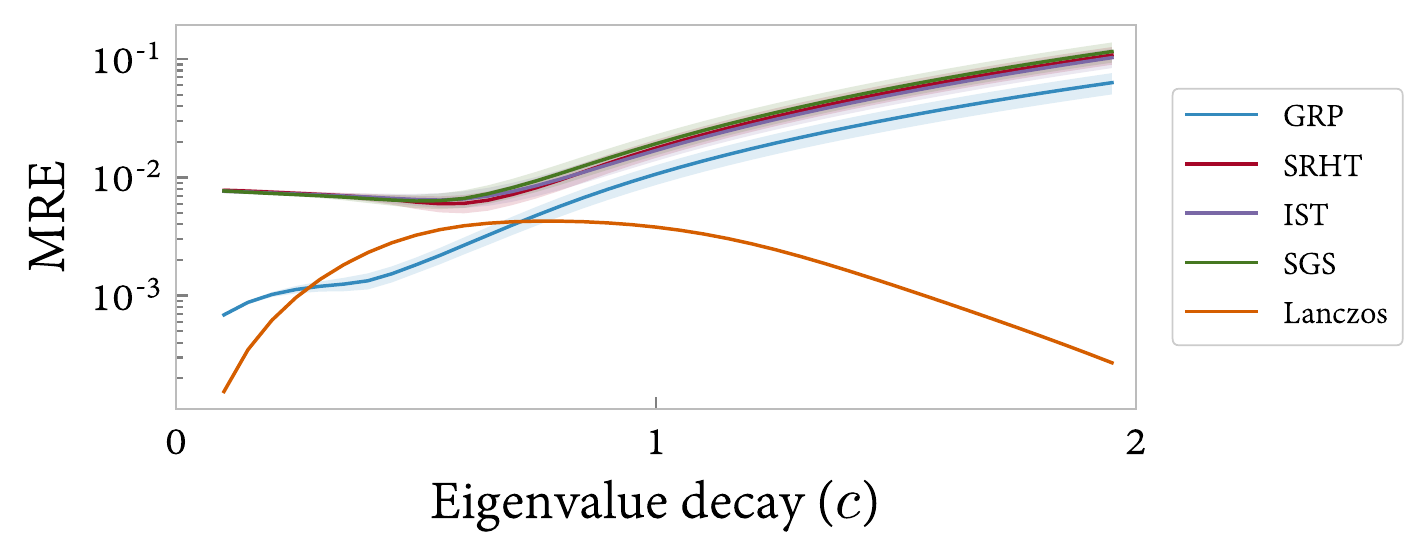}
	\caption{$c$ versus MRE curves for entropy approximation.}
	\label{DecayExp}
\end{figure}

\begin{figure}[t]
	\small
	\centering
	\includegraphics[width=0.45\textwidth]{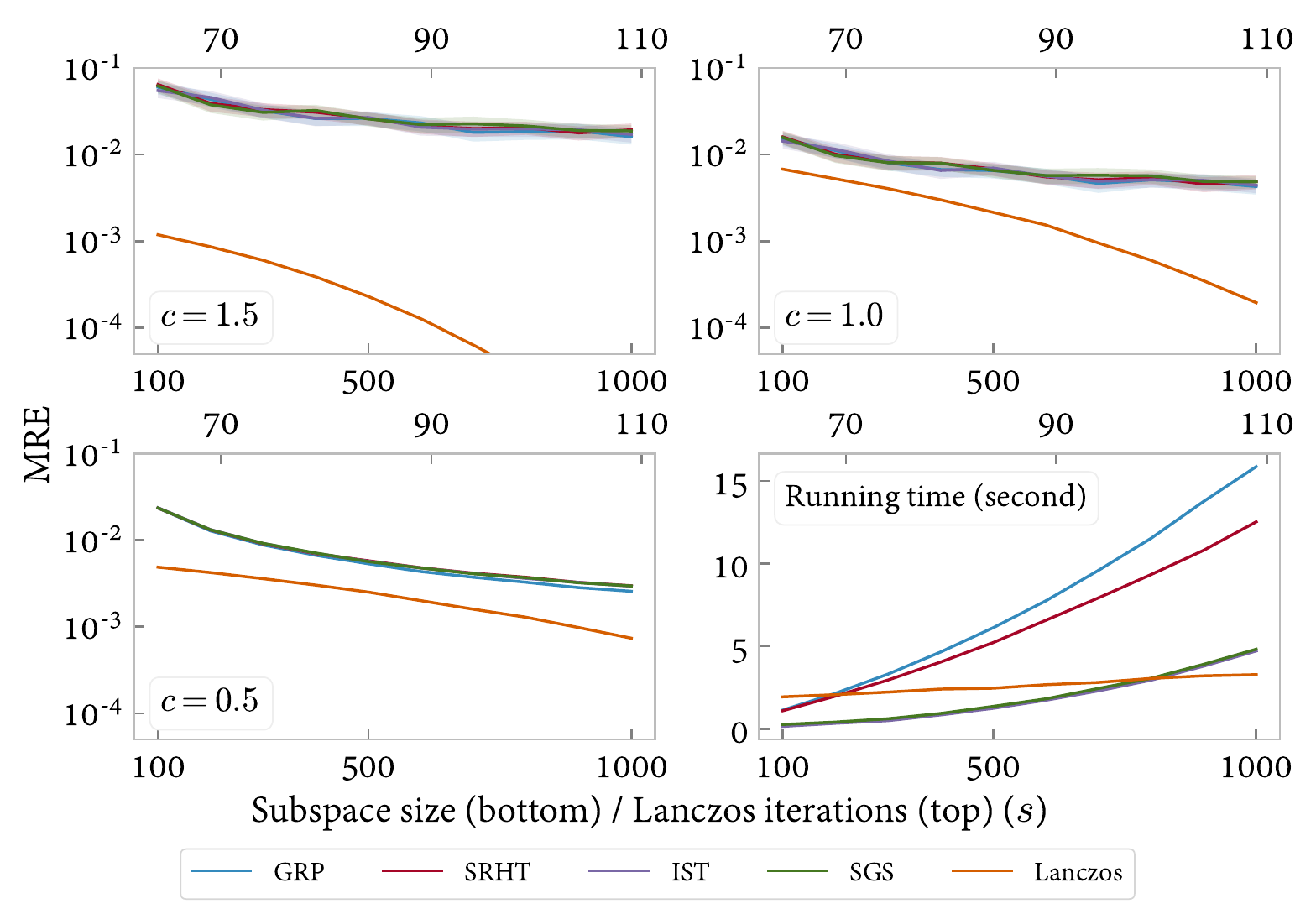}
	\caption{$s$ versus MRE curves for entropy approximation. The first three sub-figure correspond to different $c$ values, while the last sub-figure show the running time.}
	\label{SExp}
\end{figure}

\begin{table}[ht]
	\centering
	\begin{tabular}{ ccc }
		\toprule
		Dataset & $\alpha$ & $k$ \\
		\midrule
		Breast & 2.0 & 100 \\
		Semeion & 1.01 & 400 \\
		Madelon & 2.0 & 100 \\
		Krvskp & 1.01 & 200 \\
		Spambase & 2.0 & 100 \\
		Waveform & 2.0 & 100 \\
		Optdigits & 1.01 & 200 \\
		Statlog & 0.6 & 100 \\
		\bottomrule
	\end{tabular}
	\caption{Hyper-parameter selection results of $\alpha$ and $k$ in feature selection experiment.}
	\label{tbl_hyper}
\end{table}

\begin{figure}[t]
	\small
	\centering
	\includegraphics[width=0.3\textwidth]{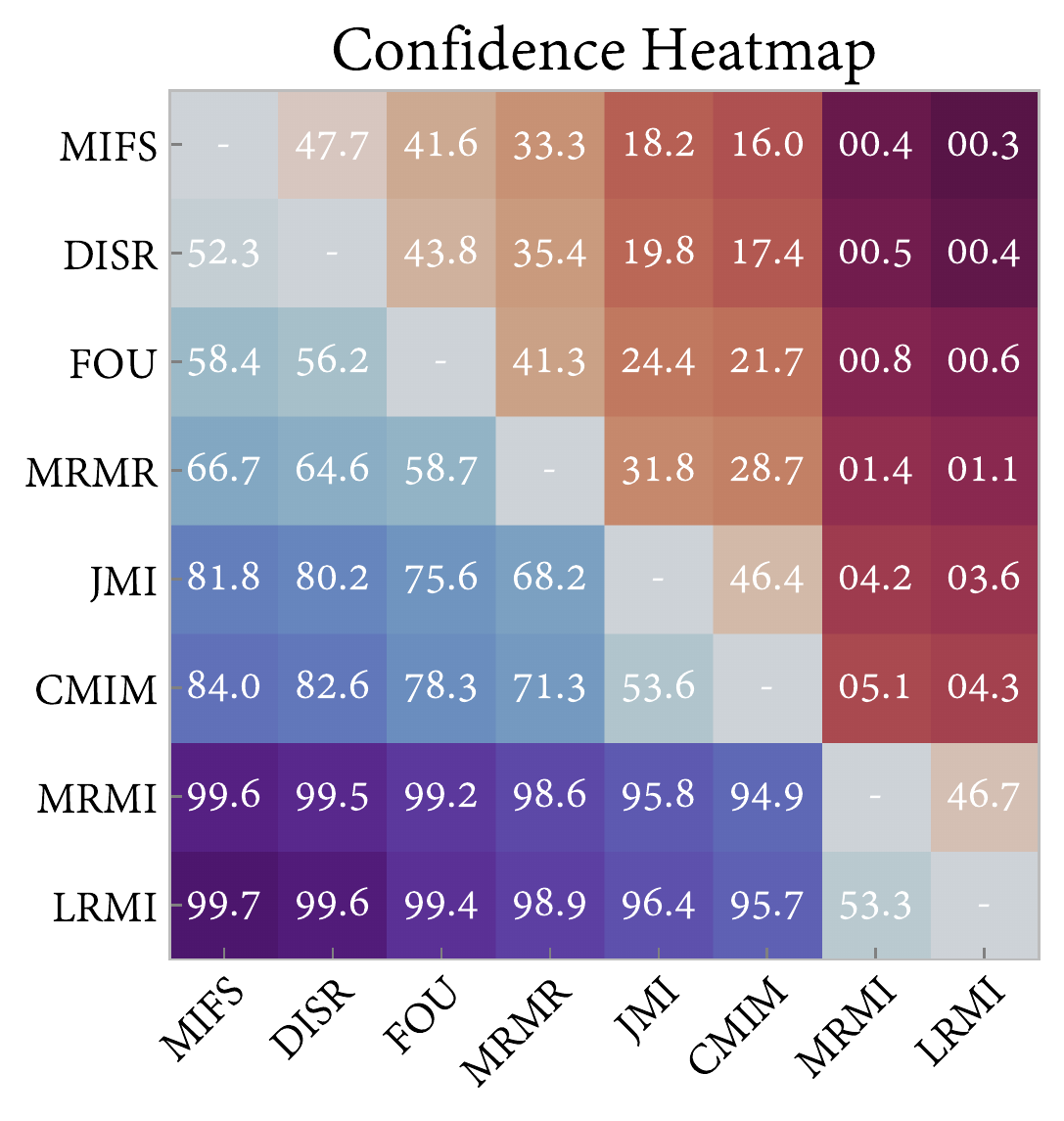}
	\caption{Confidence of significant outperforming (\%) for different feature selection methods.}
	\label{fig_confidence}
\end{figure}

\begin{figure*}[t]
	\small
	\centering
	\includegraphics[width=0.9\textwidth]{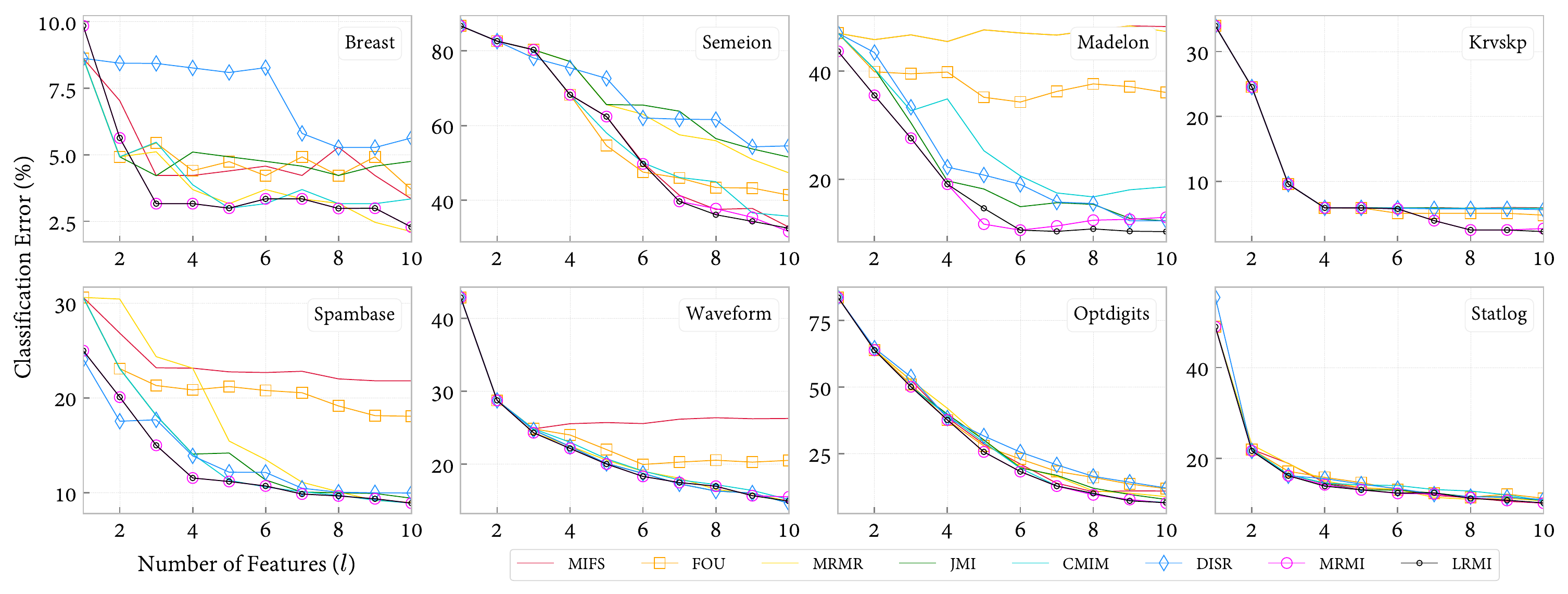}
	\caption{Number of Features ($l$) versus Classification Error (\%) curves for different feature selection methods.}
	\label{fig_selection}
\end{figure*}

\section{Proof of Main Results}
\subsection{Proof of Proposition \ref{prop_lowrank}}
\begin{proof}
	For (a): Let $\A = \U\LLambda\U^\top$ be the eigenvalue decomposition of $\A$, then $\P\U$ is a unitary matrix and $\lambda_i(\A) = \lambda_i(\P\A\P^\top)$ for all $i \in [1,k]$.
	
	For (b): When $p > 0$, $\tr\prn*{(pL_k(\A))^\alpha} = p^\alpha \cdot \tr(L_k^\alpha(\A)) > 0$, then (b) follows by the continuity of the logarithm function.
	
	For (c): Notice that $\S_\alpha^k(\A) = \S_\alpha(L_k(\A))$, where $\tr(L_k(\A)) = 1$ and $\lambda_i(L_k(\A)) \in [0,1]$ for all $i \in [1,n]$. Then we have $\tr(L_k^\alpha(\A)) \ge 1$ when $\alpha \in (0, 1)$ and $\tr(L_k^\alpha(\A)) \le 1$ when $\alpha > 1$, which further implies that $\S_\alpha^k(\A) \ge 0$.
	
	Let $f(x) = x^\alpha$, it is obvious that $f$ is concave when $\alpha \in (0,1)$ and convex when $\alpha > 1$. Then by Jensen's inequality, $\tr(f(L_k(\A))) \le \tr(f(\frac{1}{n}I))$ when $\alpha \in (0,1)$ and otherwise the opposite, which further implies that $\S_\alpha^k(\A) \le \S_\alpha^k(\frac{1}{n}I)$.
	
	Moreover, it is straightforward to show that $\S_\alpha^k(\frac{1}{n}I) = \S_\alpha(\frac{1}{n}I) = \log_2(n)$.
	
	For (d): From Proposition 4.1 in \cite{giraldo2014measures} we have that $\S_\alpha(\A \otimes \B) = \S_\alpha(\A) + \S_\alpha(\B)$, therefore $\S_\alpha(L_k(\A) \otimes L_k(\B)) = \S_\alpha(L_k(\A)) + \S_\alpha(L_k(\B))$. Notice that the smaller $(n-k)^2$ eigenvalues of $L_k(\A) \otimes L_k(\B)$ are equal to $\lambda_r(\A)\lambda_r(\B)$, we have $\S_\alpha^{n^2-(n-k)^2}(L_k(\A) \otimes L_k(\B)) = \S_\alpha^k(L_k(\A)) + \S_\alpha^k(L_k(\B))$.
	
	For (e): From Proposition 4.1 in \cite{giraldo2014measures} we have that $\S_\alpha(t\A_k + (1-t)\B_k) = g^{-1}\big(tg(\S_\alpha(\A)) + (1-t)g(\S_\alpha(\B))\big)$. Notice that $\A_k = \A$ when $\tr(\A_k) = 1$, we have $\S_\alpha^k(t\A + (1-t)\B) = g^{-1}\big(tg(\S_\alpha^k(\A)) + (1-t)g(\S_\alpha^k(\B))\big)$.
	
	For (f): From the proof of Proposition 4.1 in \cite{giraldo2014measures} we have that
	\begin{equation*}
		\sum_{i=1}^t \lambda_i(\A\circ\B) \le \frac{1}{n}\sum_{i=1}^t \lambda_i(\B),
	\end{equation*}
	where $t$ is any integer in $[1,n]$. Therefore
	\begin{align*}
		\sum_{i=1}^t \lambda_i(L_k(\A\circ\B)) &\le \frac{1}{n}\sum_{i=1}^t \lambda_i(L_k(\B)),\quad \forall t \in [1,k], \\
		\sum_{i=1}^n \lambda_i(L_k(\A\circ\B)) &= \frac{1}{n}\sum_{i=1}^n \lambda_i(L_k(\B)) = \frac{1}{n},
	\end{align*}
	From the case $t = k$ we know that $\lambda_r(\A\circ\B) \ge \lambda_r(\B)/n$, therefore for any $t \in [k + 1, n]$, we have
	\begin{align*}
		\sum_{i=1}^t \lambda_i(L_k(\A\circ\B)) &= \frac{1}{n} - (n-t)\lambda_r(\A\circ\B) \\
		&\le \frac{1}{n} - \frac{n-t}{n}\lambda_r(\B) \\
		&= \frac{1}{n}\sum_{i=1}^t \lambda_i(L_k(\B)).
	\end{align*}
	Then we can prove that
	\begin{align*}
		\S_\alpha^k\prn*{\frac{\A\circ\B}{\tr(\A\circ\B)}} &= \S_\alpha\prn*{L_k\prn*{\frac{\A\circ\B}{\tr(\A\circ\B)}}}\\
		&\ge \S_\alpha(L_k(\B)) = \S_\alpha^k(\B)
	\end{align*}
	following the proof in \cite{giraldo2014measures}.
	
	For (g): From the proof of Proposition 4.1 in \cite{giraldo2014measures}, when $\A = \frac{1}{n}\1\1^\top$ and $\A = \frac{1}{n}I$, we have
	\begin{align*}
		\sum_{i=1}^t \lambda_i(\A\circ\B) &\le \frac{1}{n}\sum_{i=1}^t \lambda_i(\B)\quad \mathrm{and} \\
		\frac{1}{n}\sum_{i=1}^t \lambda_i(\A\circ\B) &\le \frac{1}{n}\sum_{i=1}^t \lambda_i(\B)
	\end{align*}
	respectively, where $t$ is any integer in $[1,n]$. Similar with the proof of (f), for these two extreme cases we can prove that
	\begin{align*}
		\sum_{i=1}^t \lambda_i(L_k(\A\circ\B)) &\le \frac{1}{n}\sum_{i=1}^t \lambda_i(L_k(\B))\quad \mathrm{and} \\
		\frac{1}{n}\sum_{i=1}^t \lambda_i(L_k(\A\circ\B)) &\le \frac{1}{n}\sum_{i=1}^t \lambda_i(L_k(\B))
	\end{align*}
	respectively. These inequalities imply that
	\begin{equation*}
		\S_\alpha\prn*{L_k\prn*{\frac{\A\circ\B}{\tr(\A\circ\B)}}} \le \S_\alpha(L_k(\A)) + \S_\alpha(L_k(\B))
	\end{equation*}
	following the proof in \cite{giraldo2014measures}.
\end{proof}

\subsection{Proof of Theorem \ref{th:robust}}
\begin{proof}
	Without loss of generality, we assume $\mu_1 \ge \mu_2 \ge \cdots \ge \mu_n$. Note that $\lambda_i$, $i \in [1,n]$ may not be monotonically decreasing. By the definition of information potential, we have
	\begin{gather*}
		\IP_\alpha(\B) = \sum_{i=1}^n \mu_i^\alpha, \\
		\IP_\alpha^k(\B) = \sum_{i=1}^k \mu_i^\alpha + (n-k) \mu_r^\alpha, \\
		\mu_r = \frac{1}{n-k}\prn*{1 - \sum_{i=1}^k \mu_i}.
	\end{gather*}
	When $\nu_i$ is small, we have the following first-order approximation:
	\begin{align*}
		\mu_i^\alpha &= (\lambda_i+\nu_i)^\alpha \\
		&= \lambda_i^\alpha + \alpha\lambda_i^{\alpha-1}\nu_i + \frac{\alpha(\alpha-1)}{2}\lambda_i^{\alpha-2}\nu_i^2 + \cdots \\
		&= \lambda_i^\alpha + \alpha\lambda_i^{\alpha-1}\nu_i + o(\nu_i).
	\end{align*}
	Therefore
	\begin{align*}
		\Var[\IP_\alpha(\B)] &= \Var[\IP_\alpha(\B) - \IP_\alpha(\A)] \\
		&= \Var\brk*{\sum_{i=1}^n \mu_i^\alpha - \lambda_i^\alpha} \\
		&= \Var\brk*{\sum_{i=1}^n \alpha\lambda_i^{\alpha-1}\nu_i + o(\nu_i)} \\
		&\approx \alpha^2\sum_{i=1}^n \Var\brk*{\lambda_i^{\alpha-1}\nu_i} \\
		&= \alpha^2\sum_{i=1}^n \sigma_i^2\lambda_i^{2(\alpha-1)}.
	\end{align*}
	Similarly, we have
	\begin{align*}
		\Var[\IP_\alpha^k(\B)] &= \Var[\IP_\alpha^k(\B) - \IP_\alpha^k(\A)] \\
		&= \Var\brk*{\sum_{i=1}^k \prn*{\mu_i^\alpha - \lambda_i^\alpha} + (n-k)(\mu_r^\alpha - \lambda_r^\alpha)} \\
		&= \Var\Bigg[\sum_{i=1}^k \prn*{\alpha\lambda_i^{\alpha-1}\nu_i + o(\nu_i)} \\
		&\qquad - \alpha(n-k)\lambda_r^{\alpha-1}\cdot \frac{1}{n-k}\sum_{i=1}^k \nu_i \Bigg].
	\end{align*}
	When $\alpha \in (0, 1)$, i.e. $\alpha-1 < 0$, we have $\lambda_r^{\alpha-1} \ge \lambda_i^{\alpha-1}$ for $i \in [1, k]$. Therefore
	\begin{align}
		\Var[\IP_\alpha^k(\B)] &\le \Var\brk*{\alpha\lambda_r^{\alpha-1}\sum_{i=1}^k \nu_i} \nonumber\\
		&= \alpha^2\lambda_r^{2(\alpha-1)}\sum_{i=1}^k\sigma_i^2 \nonumber\\
		&\le \alpha^2\sum_{i=k+1}^n \sigma_i^2\lambda_i^{2(\alpha-1)} \frac{\sum_{i=1}^k\sigma_i^2}{\sum_{i=k+1}^n\sigma_i^2} \label{eq:jensen}\\
		&\le \Var[\IP_\alpha(\B)]. \nonumber
	\end{align}
	(\ref{eq:jensen}) follows by Jensen's inequality using the fact that $\lambda_r = \frac{1}{n-k} \sum_{i=k+1}^n \lambda_i$ and $\sigma_i$ are non-negative, since the function $f(x) = x^{2(\alpha-1)}$ is convex.
	
	Otherwise when $\alpha > 1$, we have $\lambda_r^{\alpha-1} \le \lambda_i^{\alpha-1}$ for $i \in [1, k]$. Therefore
	\begin{align*}
		\Var[\IP_\alpha^k(\B)] &\le \Var\brk*{\alpha\sum_{i=1}^k \lambda_i^{\alpha-1}\nu_i} \\
		&= \alpha^2\sum_{i=1}^k\sigma_i^2\lambda_i^{2(\alpha-1)} \\
		&\le \Var[\IP_\alpha(\B)].
	\end{align*}
	This completes the proof.
\end{proof}

\subsection{Uniqueness of Low-rank R\'enyi's Entropy}
Let $\S_\alpha^k(\A)$ be a measure of entropy defined on the largest $k$ eigenvalues of $\A$. Then $\S_\alpha^k(\A)$ must adopt some strategy to build a probability distribution upon known eigenvalues, i.e. let the summation of all eigenvalues be exactly $1$, otherwise $\S_\alpha^k(\A)$ will not be continuous at $\alpha=1$. One choice is to adopt some strategy to complement the missing eigenvalues. Let $L_k(\A)$ be the complemented matrix, we have $\lambda_i(L_k(\A)) = \lambda_i(\A), \forall i \in [1,k]$, $\lambda_n(L_k(\A)) \le \cdots \le \lambda_{k+1}(L_k(\A)) \le \lambda_k(\A)$ and $\tr(L_k(\A)) = 1$.

Let $F_\A(t)$ be the CDF of $\A$: $F_\A(t) = \sum_{i=1}^t \lambda_i(\A)$, and let $\lambda_r(\A) = \frac{1}{n-k}(1 - \sum_{i=1}^k \lambda_i(\A))$. Then we have
\begin{equation} \label{eq_unique_lower}
	F_{L_k(\A)}(t) \ge \sum_{i=1}^k \lambda_i(\A) + (t-k) \lambda_r(\A)
\end{equation}
for all $t \in [k+1,n]$ since the function $F_\A$ is always concave. Let $\B \in \mathbb{R}^{n \times n}$ be a PSD matrix satisfying
\begin{equation*}
	\sum_{i=1}^t \lambda_i(\A) \le \sum_{i=1}^t \lambda_i(\B)
\end{equation*}
for all $t \in [1,n]$, then in order to maintain the triangle inequality (axiom (f) and (g)), The function $F_\A$ must satisfy $F_{L_k(\A)}(t) \le F_{L_k(\B)}(t)$, $\forall t \in [k+1,n]$. Construct $\B$ by letting $\lambda_1(\B) = 1 - (n-1)\lambda_r(\A)$ and $\lambda_2(\B) = \cdots = \lambda_n(\B) = \lambda_r(\A)$, then combining with the fact that $\lambda_i(L_k(\B)) \le \lambda_k(\B)$, $\forall i \in [k+1,n]$, we have
\begin{equation} \label{eq_unique_upper}
	F_{L_k(\A)}(t) \le F_{L_k(\B)}(t) = \sum_{i=1}^k \lambda_i(\A) + (t-k) \lambda_r(\A).
\end{equation}
Eq. (\ref{eq_unique_lower}) and (\ref{eq_unique_upper}) together imply that taking $\lambda_i(L_k(\A)) = \lambda_r(\A)$ for all $i \in [k+1,n]$ is the only choice that fulfills all axioms in Proposition \ref{prop_lowrank}.

Another reasonable choice to normalize the probability distribution is to scale the known largest $k$ eigenvalues:
\begin{equation*}
	\S_\alpha^k(\A) = \frac{1}{1-\alpha}\log_2 \prn*{\sum_{i=1}^k \prn*{\frac{\lambda_i(\A)}{\sum_{i=1}^n \lambda_i(\A)}}^\alpha},
\end{equation*}
or
\begin{equation*}
	\S_\alpha^k(\A) = \frac{1}{1-\alpha}\log_2 \prn*{\frac{\sum_{i=1}^k \lambda_i^\alpha(\A)}{\sum_{i=1}^n \lambda_i(\A)}}.
\end{equation*}
However, these methods do not fulfill the triangle inequality, i.e. we cannot infer $\S_\alpha^k(\A) \ge \S_\alpha^k(\B)$ from the condition that $F_{\A}(t) \le F_{\B}(t)$, $\forall t \in [1,k]$. This results in violations of axioms (f) and (g).

\subsection{Proof of Theorem \ref{th:rp_approx}}
We first present the $\ell_2$ embedding results for RGP, SRHT, IST and SGS in Lemma \ref{lemma_RGP}, \ref{lemma_hadamard}, \ref{lemma_sparse} and \ref{lemma_bipartite} respectively, where the dimension of embedding subspace is given to guarantee the $\epsilon$ error. Lemma \ref{lemma_permutation} presents the permutation bound for symmetric positive definite matrix. All these theoretical results are helpful to our proof.
\begin{lemma} \label{lemma_RGP}
	\cite{foucart13} Let $\U \in \mathbb{R}^{n \times k } $ such that $\U^\top \U = \I_k$ and $\P \in \mathbb{R}^{n \times s}$ constructed by GRP. Then, with probability at least  $1 - \delta$, 
	\begin{equation*}
		\left\|   \U^\top \P \P^\top \U - \I_k \right\|_2 \leq \epsilon,
	\end{equation*}
	by setting $s = \mathcal{O} \left(  k + \log(1/\delta) /\epsilon^2 \right)$.
\end{lemma}

\begin{lemma} \label{lemma_hadamard}
	\cite{drineas2012fast} 	Let $\U \in \mathbb{R}^{n \times k } $ such that $\U^\top \U = \I_k$ and $\P \in \mathbb{R}^{n \times s}$ constructed by SRHT. Then, with probability at least 0.9, 
	\begin{equation*}
		\left\|  \frac{n}{k} \U^\top \P \P^\top \U - \I_k \right\|_2 \leq \epsilon,
	\end{equation*}
	by setting $s = \mathcal{O} \left( (k+ \log n) \frac{\log k}{\epsilon^2}\right)$.
\end{lemma}

\begin{lemma} \label{lemma_sparse}
	\cite{woodruff2014sketch} Let $\U \in \mathbb{R}^{n \times k } $ such that $\U^\top \U = \I_k$ and $\P \in \mathbb{R}^{n \times s}$ constructed by IST. Then, with probability at least 0.9, 
	\begin{equation*}
		\left\|   \U^\top \P \P^\top \U - \I_k \right\|_2 \leq \epsilon,
	\end{equation*}
	by setting $s = \mathcal{O} \left(  k^2 /\epsilon^2 \right)$.
\end{lemma}

\begin{lemma} \label{lemma_bipartite}
	\cite{hu2021sparse} Let $\U \in \mathbb{R}^{n \times k } $ such that $\U^\top \U = \I_k$ and $\P \in \mathbb{R}^{n \times s}$ constructed by SGS. Then, with probability at least $1 - \delta$, 
	\begin{equation*}
		\left\|   \U^\top \P \P^\top \U - \I_k \right\|_2 \leq \epsilon,
	\end{equation*}
	by setting
	\begin{align*}
		s &= \mathcal{O} \left(k\log(k/\delta\epsilon)/\epsilon^2 \right), \\
		p &= \mathcal{O} \left(\log(k/\delta\epsilon) /\epsilon \right).
	\end{align*}
\end{lemma}

\begin{lemma} \label{lemma_permutation}
	\cite{demmel1992jacobi}	Let $\D \G \D$ be a symmetric positive definite matrix such that $\D$ is a diagonal matrix and $\G_{ii}= 1$ for all $i$. Let $\D \E \D$ be a permutation matrix such that $\| \E\|_2 < \lambda_{\min}(\G)$. Let $\lambda_i$ be the $i$-th eigenvalue of $\D \G \D$ and $\hat{\lambda}_i$ be the $i$-th eigenvalue of $\D(\G+\E)\D$. Then, for all $i$,
	\begin{equation*}
		|\lambda_i - \hat{\lambda}_i| \leq \frac{\| \E\|_2}{\lambda_{\min}(\G)}.
	\end{equation*}
\end{lemma}

The following proposition shows that if we can bound each eigenvalue of $\A$ to absolute error $\epsilon$, we have an absolute bound for $\S_\alpha(\A)$.

\begin{proposition} \label{prop_renyi_bound}
	Let $\A$ and $\hat{\A}$ be positive definite matrices with eigenvalues $\lambda_i$ and $\hat{\lambda}_i$, $i \in [1,n]$ respectively, such that for each $i \in [1,n]$, $\abs{\lambda_i - \hat{\lambda}_i} \le \epsilon$, then
	\begin{equation*}
		\abs{\S_\alpha(\A) - \S_\alpha(\hat{\A})} \le \abs*{\frac{\alpha}{1-\alpha} \log_2 \prn*{1-\frac{\epsilon}{\lambda_n}}}.
	\end{equation*}
\end{proposition}
\begin{proof}
	Let $\lambda_n > 0$ be the smallest eigenvalue of $\A$ and let $\epsilon_0 = \epsilon/\lambda_n$, then we have $\abs{\lambda_i - \hat{\lambda}_i} \le \epsilon_0 \lambda_i$ for each $i \in [1,n]$. Observe that when $\alpha < 1$,
	\begin{align*}
		\S_\alpha(\hat{\A}) &= \frac{1}{1-\alpha} \log_2 \prn*{\sum_{i=1}^n \hat{\lambda}_i^\alpha} \\
		&\ge \frac{1}{1-\alpha} \log_2 \prn*{(1-\epsilon_0)^\alpha \sum_{i=1}^n \lambda_i^\alpha} \\
		&= \frac{1}{1-\alpha} \log_2 \prn*{\sum_{i=1}^n \lambda_i^\alpha} + \frac{\alpha}{1-\alpha} \log_2(1-\epsilon_0) \\
		&= \S_\alpha(\A) + \frac{\alpha}{1-\alpha} \log_2(1-\epsilon_0).
	\end{align*}
	Similarly, we have
	\begin{align*}
		\S_\alpha(\hat{\A}) &= \frac{1}{1-\alpha} \log_2 \prn*{\sum_{i=1}^n \hat{\lambda}_i^\alpha} \\
		&\le \frac{1}{1-\alpha} \log_2 \prn*{(1+\epsilon_0)^\alpha \sum_{i=1}^n \lambda_i^\alpha} \\
		&= \frac{1}{1-\alpha} \log_2 \prn*{\sum_{i=1}^n \lambda_i^\alpha} + \frac{\alpha}{1-\alpha} \log_2(1+\epsilon_0) \\
		&= \S_\alpha(\A) + \frac{\alpha}{1-\alpha} \log_2(1+\epsilon_0).
	\end{align*}
	We can get the same results for the other case when $\alpha > 1$, which finishes the proof.
\end{proof}

\begin{proof}[Proof of Theorem \ref{th:rp_approx}]
	Note that $\lambda_{\min}(\G)$ in the Lemma \ref{lemma_permutation} is a real, strictly positive number since $\G$ is positive definite and the fact $0 \leq \|E\|_2  \lambda_{\min}(\G)$. Now consider the matrix $\A \P \P^\top \A^\top$, we will show that the singular values of $\A \P \P^\top \A$ are sufficient approximation to that of  $\A \A^\top$ by the permutation theory presented in Lemma \ref{lemma_permutation}.
	
	Let $\lambda_i$, $i \in [1,n]$ be the eigenvalues of the positive definite kernel matrix $\A$, $\hat{\lambda}_i$ be their approximations and $\A = \PPhi \SSigma \PPhi^\top$ be the eigenvalue decomposition $\A$. Since $\PPhi$ is an orthogonal matrix, we have that the eigenvalues of $\PPhi \SSigma \PPhi^\top \P \P^\top \PPhi \SSigma \PPhi^\top$ are equal to the eigenvalues of $\SSigma \PPhi^\top \P \P^\top \PPhi \SSigma$. Let $\SSigma_k$ be the $k \times k$ diagonal matrix containing the $k$ largest eigenvalues of $\A$ and $\PPhi_k$ be the matrix containing the corresponding eigenvectors, then $\lambda_i^2$, $i \in [1,k]$ are the eigenvalues of matrix $\SSigma_k\I_k\SSigma_k$, and $\hat{\lambda}_i^2$, $i \in [1,k]$ are the eigenvalues of matrix $\SSigma_k \PPhi_k^\top \P \P^\top \PPhi_k \SSigma_k$ (since the first $k$ singular values of $\SSigma_k \PPhi_k^\top \P$ are equal to those of  $\PPhi \SSigma \PPhi^\top \P = \A \P$). Let $\E = \PPhi_k^\top \P \P^\top \PPhi_k - \I_k$, we know from Lemma \ref{lemma_RGP} (or Lemma \ref{lemma_hadamard}, \ref{lemma_sparse} and \ref{lemma_bipartite}) that $\|\E\|_2 \leq \epsilon_0$ with high probability. It meets the condition of Lemma \ref{lemma_permutation} since $\lambda_{\min}(\I_k) = 1$. Hence, we have
	\begin{equation*}
		|\lambda_i^2 - \hat{\lambda}_i^2 | \le \epsilon_0, \quad \forall i \in [1,k],
	\end{equation*}
	which then implies that
	\begin{equation*}
		\lambda_i - \sqrt{\lambda_i^2 - \epsilon_0} \le \abs{\hat{\lambda}_i - \lambda_i} \le \sqrt{\lambda_i^2 + \epsilon_0} - \lambda_i.
	\end{equation*}
	Since $\lambda_k$ is the smallest eigenvalue amongst $\lambda_i$, $i \in [1,k]$, we have
	\begin{align*}
		\abs{\hat{\lambda}_i - \lambda_i} &\le \lambda_k - \sqrt{\lambda_k^2 - \epsilon_0} \\
		&= \lambda_k \prn*{1 - \sqrt{1 - \frac{\epsilon_0}{\lambda_k^2}}} \\
		&\le \lambda_k \frac{\epsilon_0}{\lambda_k^2} = \frac{\epsilon_0}{\lambda_k}.
	\end{align*}
	Combining with $k \le n/2$, we have
	\begin{align*}
		\abs{\hat{\lambda}_r - \lambda_r} &= \left.\abs*{\sum_{i=1}^k \hat{\lambda}_i - \sum_{i=1}^k \lambda_i} \middle/ (n-k)\right. \\
		&\le \frac{\epsilon_0}{\lambda_k} \cdot \frac{k}{n-k} \le \frac{\epsilon_0}{\lambda_k}.
	\end{align*}
	Let $\epsilon_0 = \epsilon\lambda_k\lambda_r$ and $\B$ be a positive definite matrix with the first $k$ eigenvalues equal to $\hat{\lambda}_i$, $i \in [1,k]$ and the other $n-k$ eigenvalues equal to $\hat{\lambda}_r$. Recall that $\lambda_r$ is the smallest eigenvalue of $L_k(\A)$, by applying Proposition \ref{prop_renyi_bound}, we have
	\begin{align*}
		\abs{\S_\alpha^k(\A) - \hat{\S}_\alpha^k(\A)} &= \abs{\S_\alpha(L_k(\A)) - \S_\alpha(\B)} \\
		&\le \abs*{\frac{\alpha}{1-\alpha} \log_2 \prn*{1-\epsilon}}.
	\end{align*}
\end{proof}

\subsection{A Potential Improvement}
The upper bound of $s$ in Theorem \ref{th:rp_approx} relies on $\lambda_r$, which grows large if the kernel matrix $\A$ is ill-posed and $\lambda_r$ is small. Alternatively, we derive an upper bound for $s$ in terms of $n$ and $\tr(\A^\alpha)$, which is tighter for such kernel matrices.
\begin{proposition}
	Under the same conditions as Proposition \ref{prop_renyi_bound}, we have
	\begin{equation*}
		\abs{\S_\alpha(\A) - \S_\alpha(\hat{\A})} \le \begin{cases}
			\abs*{\frac{1}{1-\alpha}\log\prn*{1-\frac{n\alpha\epsilon}{1+n\epsilon}}} & \textrm{if } \alpha > 1, \\
			\abs*{\frac{1}{1-\alpha}\log\prn*{1-\frac{n\epsilon^\alpha}{tr(A^\alpha)}}} & \textrm{if } \alpha < 1.
		\end{cases}
	\end{equation*}
\end{proposition}
\begin{proof}
	\begin{align*}
		S_\alpha(\tilde{\A})-S_\alpha(\A) &= \abs*{\frac{1}{1-\alpha}\log\prn*{1-\frac{\sum_{i=1}^n\hat{\lambda}_i^\alpha - \sum_{i=1}^n\lambda_i^\alpha}{\sum_{i=1}^n\hat{\lambda}_i^\alpha}}} \\
		&\le \abs*{\frac{1}{1-\alpha}\log\prn*{1-\beta}},
	\end{align*}
	where
	\begin{align*}
		\beta &= \abs*{\frac{\sum_{i=1}^n(\lambda_i+\epsilon)^\alpha - \sum_{i=1}^n\lambda_i^\alpha}{\sum_{i=1}^n(\lambda_i+\epsilon)^\alpha}} \\
		&\le \abs*{\frac{\sum_{i=1}^n(\lambda_i+\epsilon)^\alpha - \sum_{i=1}^n\lambda_i^\alpha}{\sum_{i=1}^n\lambda_i^\alpha}}.
	\end{align*}
	When $\alpha > 1$, we have
	\begin{equation*}
		\beta \le \alpha\epsilon\abs*{\frac{\sum_{i=1}^n(\lambda_i+\epsilon)^{\alpha-1}}{\sum_{i=1}^n(\lambda_i+\epsilon)^\alpha}} \le \frac{n\alpha\epsilon}{1+n\epsilon},
	\end{equation*}
	where the last step takes equality if and only if $\lambda_1 = \cdots = \lambda_n = \frac{1}{n}$. Otherwise when $\alpha < 1$,
	\begin{equation*}
		\beta \le \abs*{\frac{\sum_{i=1}^n\epsilon^\alpha}{\sum_{i=1}^n\lambda_i^\alpha}} = \frac{n\epsilon^\alpha}{tr(A^\alpha)}.
	\end{equation*}
	One can upper bound $S_\alpha(A) - S_\alpha(\tilde{A})$ through the same strategy, which finishes the proof.
\end{proof}

\subsection{Proof of Theorem \ref{th:lanczos_approx}}
The following lemma gives the convergence rate of the Lanczos algorithm:
\begin{lemma} \label{lemma_lanczos}
	\cite{saad1980rates} Let $\q$ be the initial vector, $\lambda_i$ be the $i$-th largest eigenvalue of $\A$ with associated eigenvector $\pphi_i$ such that $\ang{\pphi_i, \q} \ne 0$, $\hat{\lambda}_i$ be the corresponding approximation of $\lambda_i$ after $s$ steps of Lanczos iteration, and assume that $\hat{\lambda}_{i-1} > \lambda_i$. Let
	\begin{align*}
		\gamma_i &= 1 + 2\frac{\lambda_i - \lambda_{i+1}}{\lambda_{i+1} - \lambda_n}, \\
		K_i &= \begin{cases}
			\prod_{j=1}^{i-1} \frac{\hat{\lambda}_j - \lambda_n}{\hat{\lambda}_j - \lambda_i}, & i > 1 \\
			1, & i = 1
		\end{cases},
	\end{align*}
	then
	\begin{equation*}
		0 \le \lambda_i - \hat{\lambda}_i \le (\lambda_i - \lambda_n) \cdot \prn*{\frac{K_i}{T_{s-i}(\gamma_i)} \tan\ang{\pphi_i, \q}}^2,
	\end{equation*}
	where $T_i(x) = \frac{1}{2} \brk*{(x+\sqrt{x^2-1})^i + (x-\sqrt{x^2-1})^i}$ is the Chebyshev polynomial of the first kind of degree $i$.
\end{lemma}

\begin{proof}[Proof of Theorem \ref{th:lanczos_approx}]
	It is easy to see that $K_i$ is monotonically increasing with the increase of $i$. Let
	\begin{align*}
		\gamma &= \min_{i \in [1,k]} \gamma_i, \\
		\theta &= \max_{i \in [1,k]} \tan\ang{\pphi_i, \q}, \\
		R &= \gamma + \sqrt{\gamma^2 - 1},
	\end{align*}
	then $\forall i \in [1,k]$,
	\begin{align*}
		\lambda_i - \hat{\lambda}_i &\le \lambda_i \cdot \prn*{\frac{2\theta K_i}{R^{s-i} + R^{-(s-i)}}}^2 \\
		&\le \lambda_i \cdot 4\theta^2K_i^2 R^{-2(s-i)} \\
		&\le \lambda_i \cdot 4\theta^2K_k^2 R^{-2(s-k)}. \\
	\end{align*}
	By selecting $s = \ceil*{k + \frac{\log\prn*{4\theta^2K_k^2/\epsilon_0}}{2\log R}}$, we have that $\forall i \in [1,k]$,
	\begin{equation*}
		\abs{\lambda_i - \hat{\lambda}_i} \le \epsilon_0\lambda_i.
	\end{equation*}
	Similarly, by combining with $k \le n/2$ we have
	\begin{equation*}
		\abs{\hat{\lambda}_r - \lambda_r} \le \epsilon_0\lambda_1.
	\end{equation*}
	Let $\epsilon_0 = \epsilon\lambda_r/\lambda_1$, by applying Proposition \ref{prop_renyi_bound}, we have
	\begin{equation*}
		\abs{\S_\alpha^k(\A) - \hat{\S}_\alpha^k(\A)} \le \abs*{\frac{\alpha}{1-\alpha} \log_2 \prn*{1-\epsilon}}.
	\end{equation*}
\end{proof}

\bibliography{Entropy_AAAI}

\end{document}